\documentclass[oneside,11pt]{article} 

\renewenvironment{abstract}
  {{\centering\large\bfseries Abstract\par}\vspace{0.7ex}%
    \bgroup
       \leftskip 20pt\rightskip 20pt\small\noindent\ignorespaces}%
  {\par\egroup\vskip 0.25ex}

{\par\egroup\vskip 0.25ex}

\usepackage{algorithm,algorithmic}
\usepackage{amssymb,amsmath,wrapfig,amsthm}
\usepackage[algo2e,ruled,linesnumbered]{algorithm2e}
\usepackage{xcolor}
\usepackage{graphicx}
\usepackage{booktabs}
\usepackage{multirow}
\usepackage{hyperref}
\numberwithin{equation}{section}

\mathchardef\dash="2D
\newtheorem{theorem}{Theorem}
\newtheorem{lemma}[theorem]{Lemma}

\newtheorem{corollary}[theorem]{Corollary}
\newtheorem{proposition}[theorem]{Proposition}

\newtheorem{assumption}{Assumption}
\newtheorem{remark}[theorem]{Remark}

\renewcommand{\bar}[1]{\overline{#1}}

\renewcommand{\P}{\mathcal{P}}
\renewcommand{\H}{\mathcal{H}}
\newcommand{\F}{\mathcal{F}}
\newcommand{\W}{\mathcal{W}}

\DeclareMathOperator{\E}{\mathbb{E}}
\newcommand{\R}{\mathbb{R}}

\newcommand{\poly}{\mathrm{poly}}
\newcommand{\maj}{\mathrm{Maj}}
\newcommand{\smaj}{\mathrm{Maj{\dash}size}}
\newcommand{\err}{\mathrm{err}}

\newcommand{\A}{\mathcal{A}}
\newcommand{\D}{{D}}
\newcommand{\X}{\mathcal{X}}
\newcommand{\Y}{\mathcal{Y}}

\newcommand{\DX}{{\D_{X}}}
\newcommand{\PC}{\mathcal{P}^{{C}}}
\newcommand{\PL}{\mathcal{P}^{{L}}}
\newcommand{\WC}{\mathcal{W}^{{C}}}
\newcommand{\WL}{\mathcal{W}^{{L}}}
\newcommand{\tmin}{t_{\min}}
\newcommand{\tmax}{t_{\max}}
\newcommand{\assign}{\leftarrow}
\newcommand{\labeloverhead}{\Lambda_{L}}
\newcommand{\compareoverhead}{\Lambda_{C}}

\newcommand{\oracletrust}{\mathcal{O}_{\mathcal{T}}}
\newcommand{\expert}{\mathcal{E}}

\newcommand{\prunelabel}{\textsc{Prune-and-Label}\xspace}
\newcommand{\complabel}{\textsc{Prune-Compare-and-Label}\xspace}
\newcommand{\filter}{\textsc{Semi-verified-Filter}\xspace}

\newcommand{\quicksort}{\textsc{Randomized Quicksort}\xspace}
\newcommand{\svquicksort}{\textsc{Semi-verified-Quicksort}\xspace}
\newcommand{\binarysearch}{\textsc{Semi-verified-BinarySearch}\xspace}
\newcommand{\test}{\textsc{Test}\xspace}
\newcommand{\Talpha}{T_{\alpha,\eta}}
\newcommand{\Tbeta}{T_{\beta,\eta}}

\renewcommand{\(}{\left(}
\renewcommand{\)}{\right)}
\renewcommand{\[}{\left[}
\renewcommand{\]}{\right]}

\newcommand{\sign}[1]{\mathrm{sign}\(#1\)}

\newcommand{\abs}[1]{\left\lvert #1 \right\rvert}

\usepackage{enumitem}

\newcommand{\citet}{\cite}
\newcommand{\citep}{\cite}

\title{Semi-Verified PAC Learning from the Crowd}
\author{
{Shiwei Zeng}\\
\texttt{szeng4@stevens.edu}\\
Stevens Institute of Technology
\and
{Jie Shen}\\
\texttt{jie.shen@stevens.edu}\\
Stevens Institute of Technology
}
\date{\today}
\begin{document}
\maketitle

\begin{abstract}
We study the problem of {\em crowdsourced PAC learning} of threshold functions. This is a challenging problem and only recently have query-efficient algorithms been established under the assumption that a noticeable fraction of the workers are {\em perfect}. In this work, we investigate a more challenging case where the {\em majority} may behave adversarially and the rest behave as the Massart noise~--~a significant generalization of the perfectness assumption. We show that under the {semi-verified model} of Charikar~et~al.~(2017), where we have (limited) access to a trusted oracle who always returns correct annotations, it is possible to PAC learn the underlying hypothesis class with a manageable amount of label queries. Moreover, we show that the labeling cost can be drastically mitigated via the more easily obtained comparison queries. Orthogonal to recent developments in semi-verified or list-decodable learning that crucially rely on data distributional assumptions, our PAC guarantee holds by exploring the wisdom of the crowd.
\end{abstract}

\section{Introduction}\label{sec:intro}



Efficient and robust learning of threshold functions is arguably one of the most important problems in machine learning, and has been broadly investigated under different label noise models, such as the random classification noise~\citep{angluin1987learning}, Massart noise~\citep{sloan1988types,massart2006risk}, Tsybakov noise~\citep{tsybakov2004optimal}, and agnostic noise~\citep{haussler1992decision,kearns1992toward}. In recent years, under certain distributional assumptions on the unlabeled data, a variety of computationally efficient algorithms have been established~\citep{awasthi2017power,zhang2020efficient,diakonikolas2020learning,diakonikolas2020polynomial,shen2021power} under the {\em standard} probably approximately correct (PAC) learning model of \cite{valiant1984theory}, where there is a {\em single} adversary who generates all the noisy labels.\footnote{In the standard PAC learning model, the label of an instance is always persistent.}

However, the standard PAC learning model of \cite{valiant1984theory} does not fully characterize the process of data collection and learning in modern machine learning applications. In fact, when constructing a large-scale data set, researchers often appeal to a crowdsourcing platform to hire {\em multiple} crowd workers for annotation, with the hope of obtaining a set of high-quality labels. To be more concrete, a crucial feature of {\em crowdsourced learning} is that for each  instance $x$, the learner has a discretion to collect a set of $k$ labels, denoted by $y_{i}(x), i=1,\dots,k$, from a large pool of heterogeneous crowd workers. Here, each single worker in the pool may behave as a certain type of label noise models in the standard PAC model, and some  workers can even collude with  others to decide which instances to corrupt among those assigned to them. A common practice to address such noisy annotations is to aggregate the labels via majority vote, which turns out to be a successful remedy provided that the majority are correct \citep{dekel2009vox,vaughan2017making,awasthi2017efficient,zeng2022crowd}.

In this paper, nevertheless, we consider learning from an extremely noisy pool of workers, where the majority might be adversarial, under which using majority voting will make the result even noisier, and more seriously, rendering PAC learning impossible. This problem, in addition to being theoretically interesting and has been broadly studied under the non-crowdsourcing setting \citep{charika2017learning}, also finds a variety of real-world applications \citep{steinhardt2016avoid,prelec2017solution,meister2018data}. For example, it is observed that when confronted with the question ``is Philadelphia the capital of Pennsylvania'', more than $60\%$ of the respondents endorse the incorrect answer with high confidence \citep{prelec2017solution}. This raises a  pressing question:

\begin{quote}
When the majority of the crowd might be malicious, can we still achieve the PAC guarantees? If the answer is positive, what is the price for it?
\end{quote}


\subsection{Warmup: reduction to standard PAC learning}

We begin our exploration by considering a quite fundamental problem setup: the underlying hypothesis class is the one of homogeneous halfspaces among which there exists a hypothesis that correctly classifies all  instances~\citep{rosenblatt1958perceptron,valiant1984theory}. For each instance we assign one worker for labeling. 

Formally, let $\X \subset \R^d$ be the instance space, $\Y = \{-1, 1\}$ be the label space, and $D$ be the marginal distribution on $\X$. The class of homogeneous halfspaces is given by $\H_{\mathrm{hs}} = \{h: x \mapsto \sign{w\cdot x}, \ w\in \R^d\}$, and the true label of $x$ is given by $h^*(x)$ for some $h^* \in \H_{\mathrm{hs}}$. The error rate of any hypothesis $h$ is defined as $\err_{D}(h) = \Pr_{x \sim \D}\big(h(x) \neq h^*(x)\big)$. The learner has access to a large pool of crowd workers who provide labels, and we denote by $\PL$ the uniform distribution on it. Given $x \in \X$, the quantity $\eta_t(x) := \Pr(h_t(x) \neq h^*(x) \mid x)$ characterizes how likely the learner gets an incorrect label when having worker $t$ annotate the instance $x$, where the probability is taken over the internal random bits of the worker.

To further simplify the problem, we assume that a $(1-\alpha)$-fraction of workers may be adversarial but the remaining $\alpha$-fraction are {\em perfect} in the sense that they always provide correct labels across the domain.
\begin{assumption}\label{as:perfect}
With probability $1-\alpha$ over the draw of a worker $t$, $t$ is adversarial, i.e. $\eta_t(x)$ can be an arbitrary function; otherwise, $t$ is perfect, i.e. $\eta_t(x) \equiv 0,\forall x \in \X$.
\end{assumption}



We aim to approach the problem by reducing it to some well-studied standard PAC learning model, such that there exists a learner with non-trivial PAC guarantees. One promising solution is the {\em agnostic learner}, which works under the scenario where an adversary can corrupt an arbitrary fraction of the labels in an adversarial manner~\citep{haussler1992decision,kearns1992toward,kalai2008agnostically,dia2021agnostic}. We give the following theorem using results from \citet{kalai2008agnostically,dia2021agnostic}.

\begin{theorem}\label{thm:agnostic}
Suppose Assumption~\ref{as:perfect} is satisfied, and further assume the maginal distribution $\D$ is standard Gaussian. Given a set of $\Theta\big(d^{\poly({1}/{\epsilon})}\big)$ independent instances from $\D$, each of which is labeled by one worker, there exists an algorithm that runs in time $O\big(({d}/{\epsilon})^{\poly({1}/{\epsilon})}\big)$ and learns $\H_{\mathrm{hs}}$ by returning a hypothesis with error rate $\leq (1-\alpha) +\epsilon$.
\end{theorem}

\begin{remark}\label{rmk:agnostic}
There are a few limitations of such method. First, the error rate is non-vanishing (as $\alpha < 1/2$), though this is the best one can hope for without extra information. Second, such reduction scheme only works for the very special case of learning homogeneous halfspaces with respect to Gaussian marginals, and with quite high computational complexity and sample complexity. Last, the success of the reduction hinges on the strong assumption that an $\alpha$ fraction of the workers are perfect. 
\end{remark}

In this paper,  we show that it is possible to break all of the above barriers by leveraging the wisdom of the crowd and by querying a trusted oracle with a  small number of times.


\subsection{Main results}

Observe that using a single worker to annotate an instance is not a common practice. Thus, from now on, we consider the more realistic scenario where the learner would assign an instance to multiple workers. An immediate caveat in our setting is that, the majority vote becomes even noisier. As a remedy, we follow a recent research line of \citet{steinhardt2016avoid,awasthi2017efficient,meister2018data} and assume that there exists a {\em trusted oracle} $\oracletrust$ who always provides the correct annotation for any instance that we query on, and we call a query to $\oracletrust$ as a verified query. To make it practical, we must restrict the number of verified queries the learner can make. We now consider a more general hypothesis class $\H$ of threshold functions, i.e. $\H = \{h: x \mapsto \sign{f(x)}, f \in \F\}$ where $\F = \{f: \X \rightarrow \R\}$; e.g. $f(x) = w \cdot x$ recovers the class of linear functions and hence $\H$ is the class of halfspaces. Note that the target classifier $h^*$ can be written as $h^*(x) = \sign{ f^*(x) }$ for some $f^* \in \F$.  We also drop Assumption~\ref{as:perfect} on the existence of perfect workers, and instead assume the following:

\begin{assumption}\label{as:massart-label}
With probability $1-\alpha$ over the draw of a worker $t$, $t$ is adversarial, i.e. $\eta_t(x)$ can be an arbitrary function; otherwise, $t$ is Massart, i.e. $\eta_t(x) \leq \eta, \forall x\in\X$ for some noise rate $\eta \in [0, 1/2)$.
\end{assumption}

The non-adversarial workers now act as the Massart noise \citep{sloan1988types,massart2006risk} which is a challenging semi-random model that has attracted a surge of recent research interests under the standard PAC model; see e.g. \citet{diakonikolas2019distribution,zhang2020efficient,diakonikolas2020learning}. Our model on the crowd workers is thus a mixture of adversarial and semi-random, with the adversarial consisting of the majority. The model itself is already new and challenging; to the best of our knowledge, no known results have been established under neither standard nor crowdsourced setting.

Now the principle of our algorithmic design is three-fold: 1) offering a PAC guarantee, i.e. finding a hypothesis $h$ such that $\err_{D}(h) \leq \epsilon$; 2) the number of verified queries, $m_V$, is few; and 3) the number of queries to the crowd, $m_L$, is moderate. In particular, we hope that the overhead of the algorithm, $\labeloverhead$, defined as the averaged number of crowd queries on a given instance, behaves as a constant.



We note that the last (implicit) condition is that the hypothesis class $\H$ has finite VC-dimension $d$, which is a minimum requirement for PAC learnability even without  noise \citep{kearns1994intro,anthony1999neural}.

\begin{assumption}\label{as:pac}
There exists a computationally efficient algorithm $\A_\H$ satisfying the following: for any $\epsilon, \delta \in (0, 1)$, $\A_\H$ draws $m_{\epsilon, \delta}$ {\em correctly labeled instances} and learns $\H$ by returning a hypothesis $h$ such that with probability $1-\delta$, $\err_{\D}(h) \leq \epsilon$. We call $\A_\H$ a {\em realizable PAC learner}.
\end{assumption}

\begin{remark}
It is known that for many interesting hypothesis classes, such as the one of polynomial threshold functions, such a realizable PAC learner does exist.
In our algorithm and analysis, we will make use of the well-known fact that it suffices to pick
\begin{equation}\label{eq:m-eps-delta}
m_{\epsilon,\delta} = K \cdot \frac{1}{\epsilon} \cdot \Big( d \log\frac{1}{\epsilon} + \log\frac{1}{\delta} \Big),
\end{equation}
where $K \geq 1$ is some absolute constant and $d$ is the VC-dimension of the class $\H$ \citep{kearns1994intro,anthony1999neural}. 
\end{remark}

\begin{theorem}[Theorem~\ref{thm:main_labelonly}, informal]\label{thm:informal_label}
Suppose that Assumptions~\ref{as:massart-label} and \ref{as:pac} are satisfied, and $\eta<\frac{\alpha}{16}$. There exists a polynomial-time algorithm (Algorithm~\ref{alg:boost}) that PAC learns $\H$. In addition, $\labeloverhead = {O}_{\alpha}(1)$ and $m_V = O_{\alpha}(1)$. 
\end{theorem}

\begin{remark}
The above  result addresses all the shortcomings  of the naive reduction approach (see Remark~\ref{rmk:agnostic}). Moreover, when the fraction of adversarial workers is a large constant, say $80\%$, our algorithm still succeeds, and more importantly, only makes a constant number of verified queries to the trusted oracle as well as a constant overhead~--~as to be clear later, a constant overhead means the total number of label queries to the crowd is only a constant multiple of the one needed in the noise-free PAC model of Assumption~\ref{as:pac}. Roughly speaking, we give an algorithm that almost does not suffer any performance loss from  the extremely noisy crowd.
\end{remark}

\begin{remark}\label{rmk:T}
The most relevant prior work is \cite{awasthi2017efficient}, where they also studied crowdsourced PAC learning in the presence of large fraction of adversarial workers. Yet, their analysis crucially relies on the existence of an $\alpha$ fraction perfect workers (i.e. Assumption~\ref{as:perfect}). When $\alpha > 1/2$, it is possible to relax their perfectness assumption to the Massart noise model as we have done. However, the case $\alpha \leq 1/2$ is more subtle and requires additional treatments (and this is one of our technical contributions).
\end{remark}


On top of Theorem~\ref{thm:informal_label}, we address the problem of learning with pairwise comparisons, where a worker is asked questions of the form ``given two instances, which one is more likely to be positive''. The motivation is that on a crowdsourcing platform, experimenters often make a same payment per annotation, and pairwise comparison queries are relatively easier to respond than label queries in many applications, thus may ease the process of data acquisition. For example, there are restaurants that people feel neither like or dislike, but comparing the preference to two restaurants might be easy. Another example is about medical diagnosis: determining whether a patient needs intensive care requires evaluation from specialists, but comparing the health status of two patients may be carried out by medical residents ~\citep{park2015preference,kane2017active,xu2017noise}.



Formally, we consider the class $\H$ of threshold functions.
For any $(x, x') \in \X \times \X$, the underlying comparison function is given by $Z^*(x, x') = \sign{f^*(x) - f^*(x')}$. Denote by $\PC$ the uniform distribution over the pool of workers who will provide comparison tags, and a worker $t$ is defined by a comparison function $Z_{t}: \X \times \X \rightarrow \{-1, 1\}$.

\begin{assumption}\label{as:massart-comp}
When providing comparison tags, a $(1-\beta)$ fraction of the workers may behave adversarially, and the rest $\beta \leq 1/2$ fraction are such that $\Pr(Z_t(x, x') \neq Z^*(x, x') \mid (x, x')) \leq \eta$ for some $\eta \in [0, 1/2)$.
\end{assumption}

Let $m_C$ be the total number of comparison queries to the crowd, $\compareoverhead$ the comparison overhead, and $m_V$ be the total number of verified label and comparison queries. Our second set of main results is a {\em label-efficient} crowdsourced PAC algorithm.

\begin{theorem}[Theorem~\ref{thm:main_bothAlphaBeta}, informal]\label{thm:informal_both}
Suppose that Assumptions~\ref{as:massart-label}, \ref{as:pac}, and \ref{as:massart-comp} are satisfied, and $\eta < \frac{\min(\alpha, \beta)}{16}$. Given any $\epsilon, \delta \in (0, 1)$, there exists a polynomial-time algorithm (Algorithm~\ref{alg:boost-comp}) that PAC learns $\H$. In addition, $\labeloverhead = o_{\alpha,\beta}(1)$, $\compareoverhead = O_{\alpha,\beta}(1)$, and $m_V = O_{\alpha,\beta}(1)$. 
\end{theorem}


\begin{remark}\label{rmk:onlyAlpha}
This is the first label-efficient algorithm that works in the semi-verified crowdsourced learning setting where adversarial workers can form a strong majority. Moreover, both the overheads and the number of verified data are independent of the sample size, meaning that the algorithm is query-efficient even for large-scale learning problems. The label efficiency follows from that $\labeloverhead$ may tend to zero; alternatively, it also follows from the fact that $m_L = \tilde{O}(\log\frac{d}{\epsilon})$, where in label-only learning, this quantity scales as $\tilde{O}(\frac{d}{\epsilon})$ (Theorem~\ref{thm:main_labelonly}). We note that a label-efficient crowdsourced PAC algorithm was obtained in \cite{zeng2022crowd}, but their analysis only works for $\alpha, \beta > \frac12$.
\end{remark}
\begin{remark}
The merit of comparison-based PAC learning algorithms is to tradeoff the labeling and comparison complexity \citep{xu2017noise}, rather than obtaining a lower overall query complexity. Indeed, the latter is information-theoretically impossible in general: a lower bound of $\Omega(d + 1/\epsilon)$ was shown in \cite{kane2017active}.
\end{remark}


\subsection{Overview of techniques}\label{sec:technique}

\noindent
{\bfseries Win-win pruning of adversarial workers.} \ 
When the adversarial workers form a majority in the crowd, a natural idea is to query the trusted oracle so as to identify and prune them. The main question here lies in how to select appropriate instances as testing cases such that a significant fraction of the adversarial workers can be pruned every time the trusted oracle is called. Our main contribution here is a {\em win-win argument} even when no perfect worker is present: the adversarial workers can either collude and provide incorrect annotations, for which the algorithm can identify and prune a noticeable fraction of them; or they are forced to provide correct annotations, for which we can be sure that the majority are correct. The key insight lies in a natural threshold of the majority size (the fraction of workers that agree with majority) being $(1-(1-\eta)\alpha)$, as on all accounts the adversarial workers cannot form a majority larger than that. This is also our condition for selecting the instances as testing cases, which generalizes the analysis in \citet{awasthi2017efficient}. See Section~\ref{subsec:labelonly}.

\medskip
\noindent
{\bfseries Query and label-efficient semi-verified filtering.} \ 
To obtain an algorithm that is both label and query-efficient (Theorem~\ref{thm:informal_both}), we need to resolve a more serious issue of majority voting over a constant number of workers: if the sample size increases, this majority voting can be incorrect with high probability. If we were to increase the constant number, the query complexity grows beyond what we can afford. To this end, we introduce a new testing stage to the comparison-equipped algorithm which draws a moderate set of informative samples and then queries comparison tags. This stage would give us enough information about how likely the crowd workers would form a strong majority, i.e. the majority size $\geq1-(1-\eta)\beta$ such that the majority vote is faithful. We show that this technique suffices for us to develop a label and query-efficient algorithm. Readers may refer to Section~\ref{subsec:comp-equipped} for details.

\subsection{Additional notations and definitions}\label{sec:setup}

We will  need the following quantity when designing comparison-based algorithms:
\begin{equation}\label{eq:n-eps-delta}
n_{\epsilon,\delta} := K \cdot \frac{1}{\epsilon} \cdot \Big( d \log\frac{1}{\epsilon} + \Big(\frac{1}{\delta}\Big)^{1/1000} \Big).
\end{equation}
The additive term $\big(\frac{1}{\delta}\big)^{{1/1000}}$ is due to the fact that the failure probability of randomized quicksort only decays to zero in a rate inversely polynomial in the sample size. There is nothing magic on the exponent $\frac{1}{1000}$; it can be made to be an arbitrarily small constant in the price of a constant factor of more samples. We note that we can also use merge sort in place of the quicksort to obtain a $\log(1/\delta)$ dependence. However, since the latter is broadly deployed in practice, we choose to bring it upfront and leave the analysis with merge sort to interested readers.

\medskip
\noindent
{\bfseries Query complexity, overhead.} \ 
The query complexity of a crowdsourced PAC learner is measured in three aspects: the total number of label queries from the crowd which we denote by $m_L$, that of comparison queries from the crowd which is denoted by $m_C$, and that of verified queries from the trusted expert which is denoted by $m_V$. Closely related to $m_L$ and $m_C$ are the following quantities: 
\begin{equation}
\labeloverhead := \frac{m_L}{m_{\epsilon,\delta}},\quad \compareoverhead := \frac{m_C}{m_{\epsilon,\delta}},
\end{equation}
termed {\em labeling overhead} and {\em comparison overhead}, respectively. These two quantities measure the query complexity of a crowdsourced PAC learner compared to the one needed in the noise-free standard PAC model. We say a crowdsourced PAC learner is {\em query-efficient} if $\labeloverhead + \compareoverhead + m_V= O(1)$, meaning that the overall query complexity is of the same order of the standard PAC model even facing an overwhelming fraction of adversarial workers. We say a crowdsourced PAC learner is {\em label-efficient} if $\labeloverhead = o(1)$. Our goal is to design a polynomial-time PAC learning algorithm that is both query-efficient and label-efficient; in addition, we aim to control the number of calls to the trusted expert as otherwise the problem is trivial.

Given an input $z$ and a set of workers $\W$, we define $\maj_{\W}(z)$ to be the majority vote over the annotations from the workers in $\W$, where $z$ can either be an instance or a pair of instances. We denote by $\smaj_{\W}(z)$ the fraction of the workers in $\W$ that agree with $\maj_{\W}(z)$, which is referred to as the empirical majority size. Similarly, we can consider the population version $\smaj_{\P}(z) = \max\{ \Pr_{t\sim\P}( h_t(z) = 1), \Pr_{t\sim\P}( h_t(z) = -1)\}$.
Specifically, let $h_1$, $h_2$ and $h_3$ be three classifiers. The function $\maj(h_1, h_2, h_3)$ maps any instance $x$ to the label $y=\maj_{\{h_1, h_2, h_3\}}(x)$.
Given a distribution $\P$ over workers and an annotated set $S$ (either a set with labels or one with comparison tags), denote by $\P_{\mid S}$ the distribution $\P$ conditioned on workers that agreed with the annotations in $S$; this will be useful to prune the adversarial workers.
Given any region $R\subseteq\X$, denote by $\D[R]$ the density of instances $x\in R$ under distribution $\D$.

\section{Related Works}\label{sec:related}


Crowdsourcing has been broadly utilized as a tool for data annotation with provable algorithms \citep{vaughan2017making}. More in line with this work is crowdsourcing in the presence of highly noisy workers, i.e. the majority can be adversarial.  \citet{steinhardt2016avoid,meister2018data} considered recovery of true values of a finite data set and \citet{awasthi2017efficient} examined PAC learnability which aligns more closely with our work. Our algorithm shares the merits with \citet{awasthi2017efficient} but we develop a novel analysis to address the label corruption from Massart workers, and we carefully leverage comparison queries to achieve label efficiency. We emphasize that crowdsourced PAC learning was also investigated in \citet{zeng2022crowd}, but their algorithm fails immediately when the majority are incorrect.

The  semi-verified learning model was recently proposed in \citet{charika2017learning}. This model is also known as list-decodable learning where a learner is allowed to output a finite list of hypotheses among which at least one is guaranteed to be close to the true hypothesis. 
Such model has been broadly studied under the non-crowdsourcing setting for various problems such as mean estimation~\citep{kothari2017better,dia2020list,zsLDsparse2022}, learning mixtures of Gaussian~\citep{dia2018list}, linear regression~\citep{karmalkar2019list,raghavendra2020list}, and subspace recovery~\citep{bakshi2020list,raghavendra2020subspace}. We note, however, these works consider a setting where both instances and labels can be corrupted adversarially, while our work assumes that only the labels are noisy (since in crowdsourcing, the data curator often has control on the instances).


Learning with comparison queries has been extensively investigated in ranking~\citep{jamieson2011active,shah2016stochastic,shah2017simple,falahatgar2017max,shah2019feel,agarwal2020rank}, best-$k$ item selection~\citep{falahatgar2017max,ren2020sample}, and recommendation systems~\citep{furnkranz2010prefer,falahatgar2018limit}. PAC learning with pairwise comparisons is investigated recently~\citep{kane2017active,xu2017noise,hopkins2020noise}; yet, the way that we utilize comparison queries is quite different from these works and their guarantees were established under the non-crowdsourcing model.

\section{Main Algorithms and Guarantees}\label{sec:algorithm}

In this section, we present our main algorithms, one that uses label-only queries (Algorithm~\ref{alg:boost}) and the other that  combines label and comparison queries (Algorithm~\ref{alg:boost-comp}). In the label-only learning model, we will build upon the algorithm of \citet{awasthi2017efficient} and show how a new pruning method (specifically, a new pruning threshold) can be leveraged into their algorithm to obtain Theorem~\ref{thm:informal_label}. In the comparison-based learning model, we will extend the algorithm of \citet{zeng2022crowd} by presenting an ensemble of new algorithmic subroutines that delicately combine verified queries, label queries, and comparison queries.

First of all, we introduce the core idea of these algorithms, assuming for now that the adversarial workers only consisting a minority. We emphasize that obtaining distribution-free PAC guarantees under label noise in the standard PAC model is highly nontrivial, and among prevalent noise models, the only known computationally efficient PAC learner works for the very basic random classification noise \citep{blum1996polynomial}, while for Massart noise or adversarial noise, hardness results have been established \citep{diamassarthard,guruswami2006hardness}. Yet, it turns out that as far as the majority is correct, by using majority vote, it is possible to PAC learn the underlying hypothesis class without distributional assumptions, where a key step is to gather sufficient labels for each instance so that with overwhelming probability, all the aggregated labels are correct, for which one can apply the realizable learner stated in Assumption~\ref{as:pac}. However, a naive label-and-train approach, as observed in \citet{awasthi2017efficient}, leads to a labeling overhead that grows with the sample size. Therefore, they proposed an elegant  algorithm that interleaves data annotation and learning. In particular, there are three phases in the main algorithm: in phase 1, it aims to learn a hypothesis $h_1$ with $\err_{D}(h_1) \leq \frac{1}{2}\sqrt{\epsilon}$; in phase 2, it identifies a distribution $D_2$ where the performance of $h_1$ is almost comparable to random guess, and learns another $h_2$ with $\err_{D_2}(h_2) \leq \frac{1}{2}\sqrt{\epsilon}$; in phase 3, it concentrates on a distribution $D_3$ where $h_1$ and $h_2$ disagree, and learns the last hypothesis $h_3$ with $\err_{\D_3}(h_3) \leq \frac{1}{2}\sqrt{\epsilon}$. The final output of the algorithm, $\hat{h}$, is a majority voting function over these three hypotheses. This is essentially the boosting algorithm of \citet{schapire1990strength} which enjoys a guarantee that $\err_{D}(\hat{h}) \leq \epsilon$. 

Our algorithm in the semi-verified setting will be built upon the above, and the primary idea is to design query-efficient pruning methods to remove most adversarial workers, such that the Massart ones form the majority. To be more detailed, we will incorporate new pruning methods in all three phases to learn the desired hypotheses, while in some cases, additional technical efforts are needed to ensure the correctness of our results. We remark that the difference to obtain  Theorem~\ref{thm:informal_both} compared to Theorem~\ref{thm:informal_label} is on how to use label queries in an economic way through the introduction of comparison queries. 




\subsection{Learning with label-only queries: Algorithm~\ref{alg:boost}}\label{subsec:labelonly}


Now let us consider that the majority are adversarial. A similar semi-verified crowdsourcing model is also studied in \cite{awasthi2017efficient} but their analysis, in particular, the success of their paradigm to prune adversarial workers, crucially relies on Assumption~\ref{as:perfect} to ensure that the perfect workers will always be retained during pruning, and thus eventually form the majority. {This is because perfect workers make no mistake, so that none of them will be removed during the pruning step and hence the fraction of the perfect ones will increase and eventually form the majority. We here leverage a key observation that, even when some perfect workers are removed, it is still possible to increase the fraction $\alpha$, so long as the fraction of the adversarial workers being removed is always strictly larger than that of the perfect workers. Hence, we can relax the assumption of {perfect workers} to  {Massart} workers (Assumption~\ref{as:massart-label}). }

The main algorithm is given in Algorithm~\ref{alg:boost}, which contains three phases corresponding to our previous discussion. The most important component of the main algorithm is \prunelabel (Algorithm~\ref{alg:prune-label}). For each given instance $x_t$, in Step~\ref{step:check_maj_size} of \prunelabel, we check whether the current fraction of Massart workers, $\alpha$, has formed the majority (note that we will update $\alpha$, thus it has potential to increase), or the  size of workers that agree with the majority vote is greater than a threshold. If either condition is satisfied, we take the majority vote as the label of $x_t$. What is subtle here is our condition on the majority size, which is essentially our pruning rule. To see why it makes sense, observe that at the beginning, the adversarial workers form a $(1-\alpha)$ fraction, while the Massart workers form the $\alpha$ fraction among which, in expectation, at most $\alpha \eta$ will return incorrect labels. Therefore, the maximum number of incorrect labels will be $1-\alpha + \alpha \eta = 1 - (1-\eta)\alpha$. So in expectation, if we see that the majority size is greater, it must be the case that the Massart workers who are committed to providing correct labels contribute to the voting, hence the majority vote is correct. It is easy to see that there is a constant factor difference in our algorithm, which is due to the conversion from expected value to high probability argument. When Step~\ref{step:check_maj_size} of \prunelabel is not satisfied, it is likely the case that majority vote gives incorrect label. This is a good chance to query the trusted oracle which guarantees elimination of a noticeable fraction of adversarial workers (and a few Massart workers). From these discussions, we can see that it is possible to achieve a win-win scenario: either adversarial workers give incorrect labels so that they can be eliminated using the verified queries, or they have to provide the correct label in order to stay in the pool until the end of the algorithm. Lastly, we show that in the presence of the Massart workers, we have to restart the main algorithm with $O(\Talpha)$ times (see Theorem~\ref{thm:main_labelonly}). Note that for \textsc{Filter} algorithm at Step~\ref{step:cite-filter} of Algorithm~\ref{alg:boost}, we can reuse the one from~\cite{awasthi2017efficient}, and we omit it due to space limit.

\begin{algorithm}[t]
	
	\caption{Main Algorithm (label only)}
	\label{alg:boost}
	
	\begin{algorithmic}[1]
		
		\REQUIRE{Parameters $\alpha$, distributions $\PL$, target error rate $\epsilon$, confidence $\delta$, noise rate $\eta$ of Massart workers, a realizable PAC learner $\A_\H$.}
		
		\ENSURE{Hypothesis $\hat{h}: \X \rightarrow \Y$ such that with probability $1-\delta$, $\err_{\D}(\hat{h})\leq \epsilon$.}
		
		\STATE $\Talpha \gets  \log_{\frac{1}{1-\eta}}(\frac{\alpha-\alpha\eta}{\alpha - \eta})$, 
		$\tau \gets 8 \Talpha$,  $\delta'\assign\frac{\delta}{\tau}$.
	
\vspace{0.05in}	
		\hspace{-12pt}{\textbf{Phase 0}}:
		
		\STATE $\bar{S_0} \assign \prunelabel(S_0,\alpha,\delta')$ for an instance set $S_0$ of size $\frac{1}{\epsilon}\log\frac{1}{\delta'}$ from $\DX$.
		
\vspace{0.05in}
		\hspace{-12pt}{\textbf{Phase 1}}:
		
		\STATE $\bar{S_1} \assign \prunelabel(S_1,\alpha, \frac{\delta'}{6})$,  for an instance set $S_1$ of size $n_{\sqrt{\epsilon}/2,\delta'/6}$ from $\DX$.
		
		\STATE $h_1  \assign \A_\H(\bar{S_1}, \frac{\sqrt{\epsilon}}{2}, \frac{\delta'}{6})$.
		
\vspace{0.05in}
		\hspace{-12pt}{\textbf{Phase 2}}:
		
		\STATE  ${S_I} \assign \textsc{Filter} (S_2, h_1,\alpha, \frac{\delta'}{12})$, for an instance set $S_2$ of size $\Theta(n_{\epsilon,\delta'/12})$ drawn from $\DX$. \label{step:cite-filter}
		
		\STATE Let $S_C$ be an instance set of size $\Theta(n_{\sqrt{\epsilon},\delta'/12})$ drawn from $\DX$.
		
		\STATE  $\overline{S_{\text{All}}} \assign \prunelabel(S_I\cup S_C,\alpha, \delta'/12)$.
		
		\STATE   $\bar{W_I} \assign \{(x,y)\in\overline{S_{\text{All}}}\mid y\neq h_1(x) \}$, $\bar{W_C} \assign \overline{S_{\text{All}}} \backslash \bar{W_I}$.
		
		\STATE Draw a sample set $\bar W$ of size $\Theta(n_{\sqrt{\epsilon}/2,\delta'/12})$ from the distribution $\frac12\bar{W_I} + \frac12 \bar{W_C}$.
		
		\STATE  $h_2 \assign \A_\H(\bar{W}, \frac{\sqrt{\epsilon}}{2}, \frac{\delta'}{6})$.
	
\vspace{0.05in}	
		\hspace{-12pt} \textbf{Phase 3}:
		
		\STATE   $\bar{S_3} \assign \prunelabel(S_3,\alpha, \delta'/6)$, for an instance set $S_3$ of size  $n_{\sqrt{\epsilon}/2,\delta'/6}$ drawn from $\DX$ conditioned on $h_1(x) \neq h_2(x)$.
		
		\STATE   $h_3 \assign \A_\H(\bar{S_3}, \frac{\sqrt{\epsilon}}{2}, \frac{\delta'}{6})$.
		
		{\bfseries return}  $\hat{h} \assign \maj(h_1, h_2, h_3)$.
		
	\end{algorithmic}
	
\end{algorithm}

\begin{algorithm}[t]
	
	\caption{\textsc{Prune-and-Label} }
	\label{alg:prune-label}
	
	\begin{algorithmic}[1]
		
		\REQUIRE {An instance set $S$, parameters $\alpha$, $\delta_1$.}
		\ENSURE {A labeled set $\bar{S}$.}
		
		\STATE  $k_0 \assign \frac{1}{2(1-\eta)^2\alpha^{2}}\cdot\log\frac{1000m}{\delta_1}$.
		
		\STATE Draw a set of $k_0$ workers $\WL$ from $\PL$.
		
		\IF{$\alpha \geq 0.7$ or $\smaj_{\WL}({x}_t)\geq 1-\frac{(1-\eta)\alpha}{4}$}{ \label{step:check_maj_size}
			\STATE	$\hat{y}_t \assign \maj_{\WL}(x_t)$. 
		}\ELSE{
			\STATE	Get a verified label $y_t^* \assign h^*(x_t)$. 
			{\bfseries Restart} Algoirthm \ref{alg:boost} with $\PL \assign \PL_{\mid ({x}_t,y_t^*)}$ and $\alpha \assign \frac{(1-\eta)\alpha}{1-\frac{(1-\eta)\alpha}{8}}$. 
		}
		\ENDIF

		{\bfseries return} $\bar{S}$.
		
	\end{algorithmic}
\end{algorithm}

\medskip
\noindent
{\bfseries Massart noise rate $\eta$.} \ 
Next, we demonstrate why the Massart noise rate $\eta$ cannot be too large. 
In our model, since roughly $\eta \alpha$ fraction of the Massart workers will provide incorrect labels, they will likely be pruned away.  Hence, denote by $\alpha_i$ the fraction of Massart workers in the pool after $i$ restartings of the main algorithm, then a necessary condition on the success of pruning is that the fraction of Massart workers is increased, i.e.
\begin{equation*}
	 \frac{(1-\eta)\alpha_{i}}{1-\frac{(1-\eta)\alpha_{i}}{8}} > \alpha_{i},
\end{equation*}
which is equivalent to $\eta < \frac{(1-\eta)\alpha_i}{8}, \forall i\geq0$. Assume the algorithm works properly and $\alpha_i\geq\alpha_{i-1},\forall i$. Then, we require that $\eta < \frac{\alpha}{16}$ by telescoping, where $\alpha$ is the fraction of Massart workers in the crowd before running the algorithm. 



\medskip
\noindent
{\bfseries Restarting times.} \ 
Note that each time the main algorithm restarts, it removes at least $\frac{(1-\eta)\alpha}{8}$ fraction of the workers, with at most an $\eta\alpha$ fraction being Massart. Therefore, denote by $N_i$ the total number of workers in the pool after $i$ restartings of the main algorithm 
\begin{align*}
\alpha_i &\geq \frac{\alpha_0 N_0\cdot\(1-\eta\)^{i}}{N_0 - \frac{(1-\eta)\alpha_0}{8}\cdot N_0 - \dots - \frac{(1-\eta)\alpha_{i-1}}{8}\cdot N_{i-1}}.
\end{align*}
By simple deductions on the above inquality, we can show that after $O\big(\log_{\frac{1}{1-\eta}} (\frac{\alpha-\alpha\eta}{\alpha - \eta})\big) =O(\Talpha )$ updates, $\alpha_i$ would surpass any given constant $C\in[0.7,1]$. In other words, the main algorithm is restarted at most $O(\Talpha)$ times, using one verified label in each restart. 
Hence, the number of verified labels required from the trusted oracle $\oracletrust$ is also bounded by $O(\Talpha)$. 

Our main result for Algorithm~\ref{alg:boost} is summarized in the following; the proof is deferred to Appendix~\ref{sec:app:proof-label}.

\begin{theorem}\label{thm:main_labelonly}
	Suppose that Assumptions~\ref{as:massart-label} and \ref{as:pac} are satisfied, and $\eta<\frac{\alpha}{16}$. Denote $\Talpha = \log_{\frac{1}{1-\eta}} (\frac{\alpha-\alpha\eta}{\alpha - \eta})$. The following holds with probability $1-\delta$. There exists an algorithm that runs in $O(\poly(d, \frac{1}{\epsilon}, {\Talpha},\frac{1}{\alpha}))$ time and learns $\H$ by returning a hypothesis $h: \X \rightarrow \Y$ with $\err_\DX(h)\leq\epsilon$. In addition, $m_L = \tilde{O}\big( \frac{d\cdot\Talpha }{\epsilon\alpha^2}\cdot \log\frac{1}{\delta}\big) $, and $m_V = O(\Talpha)$, where $\tilde{O}$ hides the logarithmic factors. Therefore, $\labeloverhead = {O}\big( \frac{\Talpha}{\alpha^2}\cdot \log^2(\Talpha) \big)$ when $\epsilon\in\big(0,\log^{-2} d\big)$. In particular, when $\alpha$ is constant, $\labeloverhead = O(1)$  and $m_V = O(1)$.
\end{theorem}

\begin{remark}
	Recall that \cite{awasthi2017efficient} studied the problem with the existence of perfect workers. Their work requires a number of $O(\frac{1}{\alpha})$ verified label from $\oracletrust$, while in our analysis, the amount is $O(\Talpha) = O(\log_{\frac{1}{1-\eta}} (\frac{\alpha-\alpha\eta}{\alpha - \eta}))$. It is not hard to show that $\lim_{\eta \rightarrow 0} \Talpha = \frac{1}{\alpha}$. Thus, our result is a new analysis on the structure of the Massart workers which guides an algorithmic refinement and derives a generalization of \cite{awasthi2017efficient}. 
\end{remark}

\subsection{Comparison-equipped algorithm: Algorithm~\ref{alg:boost-comp}}\label{subsec:comp-equipped}

In this section, the goal is to largely reduce the labeling cost using noisy pairwise comparisons from the crowd.
We note that there exists a label-efficient algorithm in~\cite{zeng2022crowd} that achieves the above tradeoff, but only works with $\alpha,\beta>\frac12$. We will use their algorithm as the starting point and we elaborate on the technical challenges and our design details.

First, since there exists $1-\beta\geq\frac12$ fraction of adversarial workers in the crowd, a natural idea is to add a pruning step when collecting comparison tags from the crowd. Similar to the analysis in Section~\ref{subsec:labelonly}, it is easy to show that if a condition $\eta<\frac{\beta}{16}$ is satisfied, the win-win pruning scheme can also apply to $\PC$ and increase the fraction of Massart workers who provide pairwise comparisons, i.e. $\beta$.
Furthermore, it requires at most $O(\Tbeta) = O\big(\log_{\frac{1}{1-\eta}} (\frac{\beta-\beta\eta}{\beta - \eta})\big)$ verified comparison tags from $\oracletrust$ to increase $\beta$ to any specified constant $C\in[0.7,1]$. This leads to our design of algorithm \complabel (see Section~\ref{subsec:complabel}).

A more serious issue with $\beta < \frac12$ roots in the filtering process of~\cite{zeng2022crowd}, which roughly speaking, will perform majority voting over a {\em constant} number of workers. In this regard, an empirical majority size may deviate far from the population one, thus the pruning fails. We also note that if we were to increase this constant number, we would suffer high query complexity. We thus develop a new testing stage that samples a large enough set of test cases and query the crowd for their majority size. If less than an $\epsilon$ fraction of the test cases have a small majority vote size, we can show that the total fraction of such  test cases in the underlying distribution $\D$ is smaller than $\epsilon$, hence the filtering process correctly filters all the desired instances except for an $\epsilon$ fraction; this is sufficient to establish a robust version of the performance guarantee. This algorithmic design is presented in \textsc{Semi-verified-Filter} (see Section~\ref{subsec:filter}).

As a result, we obtain a label-efficient semi-verified crowdsourced learning algorithm as presented in Algorithm~\ref{alg:boost-comp}.


\begin{algorithm}[t]

\caption{Main Algorithm (comparison-equipped)}
\label{alg:boost-comp}

\begin{algorithmic}[1]

\REQUIRE{Parameters $\alpha, \beta$, distributions $\PL$ and $\PC$, target error rate $\epsilon$, confidence $\delta$, noise rate $\eta$ of Massart workers, a realizable PAC learner $\A_\H$.}

\ENSURE{Hypothesis $\hat{h}: \X \rightarrow \Y$ such that with probability $1-\delta$, $\err_{\D}(\hat{h})\leq \epsilon$.}

\STATE $\Talpha \gets  \log_{\frac{1}{1-\eta}}(\frac{\alpha-\alpha\eta}{\alpha - \eta})$, $\Tbeta \gets  \log_{\frac{1}{1-\eta}}(\frac{\beta-\beta\eta}{\beta - \eta})$, $\tau \gets 8 (\Talpha + \Tbeta)$,  $\delta'\assign\frac{\delta}{\tau}$.

\vspace{0.05in}
\hspace{-12pt}{\textbf{Phase 1}}:

\STATE $\bar{S_1} \assign \complabel(S_1,\alpha,\beta, \frac{\delta'}{6})$,  for an instance set $S_1$ of size $n_{\sqrt{\epsilon}/2,\delta'/6}$ from $\DX$.

\STATE $h_1  \assign \A_\H(\bar{S_1}, \frac{\sqrt{\epsilon}}{2}, \frac{\delta'}{6})$.

\vspace{0.05in}
\hspace{-12pt}{\textbf{Phase 2}}:

\STATE  ${S_I} \assign \filter (S_2, h_1,\alpha,\beta, \frac{\delta'}{12})$, for an instance set $S_2$ of size $\Theta(n_{\epsilon,\delta'/12})$ drawn from $\DX$.

\STATE Let $S_C$ be an instance set of size $\Theta(n_{\sqrt{\epsilon},\delta'/12})$ drawn from $\DX$.

\STATE  $\overline{S_{\text{All}}} \assign \complabel(S_I\cup S_C,\alpha,\beta, \delta'/12)$.

\STATE   $\bar{W_I} \assign \{(x,y)\in\overline{S_{\text{All}}}\mid y\neq h_1(x) \}$, $\bar{W_C} \assign \overline{S_{\text{All}}} \backslash \bar{W_I}$.

\STATE Draw a sample set $\bar W$ of size $\Theta(n_{\sqrt{\epsilon}/2,\delta'/12})$ from the distribution $\frac12\bar{W_I} + \frac12 \bar{W_C}$.

\STATE  $h_2 \assign \A_\H(\bar{W}, \frac{\sqrt{\epsilon}}{2}, \frac{\delta'}{6})$.

\vspace{0.05in}
\hspace{-12pt} \textbf{Phase 3}:

\STATE   $\bar{S_3} \assign \complabel(S_3,\alpha,\beta, \delta'/6)$, for an instance set $S_3$ of size  $n_{\sqrt{\epsilon}/2,\delta'/6}$ drawn from $\DX$ conditioned on $h_1(x) \neq h_2(x)$.

\STATE   $h_3 \assign \A_\H(\bar{S_3}, \frac{\sqrt{\epsilon}}{2}, \frac{\delta'}{6})$.

{\bfseries return}  $\hat{h} \assign \maj(h_1, h_2, h_3)$.

\end{algorithmic}

\end{algorithm}

\subsubsection{\complabel}\label{subsec:complabel}

\begin{algorithm}[t]
	
	\caption{\textsc{Prune-Compare-and-Label} }
	\label{alg:label}
	
	\begin{algorithmic}[1]
		
		\REQUIRE {An instance set $S$, parameters $\alpha$, $\beta$, $\delta_1$.}
		\ENSURE {A sorted and labeled set $\bar{S}$.}
		
		\STATE  $\hat{S} \assign \svquicksort(S, \beta, \delta_1)$. 
		
		\STATE $\bar{S} \assign$ \textsc{Semi-verified-BinarySearch}$(\hat{S}, \alpha, \delta_1)$.
		
		{\bfseries return} $\bar{S}$.
		
	\end{algorithmic}
\end{algorithm}

\begin{algorithm}[t]
	
	\caption{\textsc{Semi-Verified-Quicksort}}
	\label{alg:robust_quick_sort}
	
	\begin{algorithmic}[1]
		
		\REQUIRE{An instance set $S = \{x_i\}_{i=1}^n$, $\beta$, $\delta_1$.}
		\ENSURE{A sorted list $\hat{S}$.}
		
		\STATE  $k_1 \assign \frac{1}{2(1-\eta)^2\beta^{2}}\cdot\log\frac{3006n\cdot\log n }{\delta_1}$.
		
		\STATE Apply randomized Quicksort to $S$: for each pair $(x,x')$ being compared, draw a set of $k_1$ workers $\WC$ from $\PC$ and do the following:
		
		\IF{$\beta \geq 0.7$ or $\smaj_{\WC}(x,x')\geq 1-\frac{(1-\eta)\beta}{4}$}{ \label{step:maj_size_beta}
			\STATE	Use $\maj_{\WC}(x,x')$ as the comparison tag.
		}\ELSE{
			\STATE	Get a verified comparison tag $z^* \assign Z^*(x,x')$.	{\bfseries Restart} Algorithm \ref{alg:boost} with $\PC\assign \PC_{ \mid (x, x', z^*)}$ and $\beta\assign\frac{(1-\eta)\beta}{1-\frac{(1-\eta)\beta}{8}}$. 
		}
		\ENDIF
		
		{\bfseries return} $\hat{S}$.
		
	\end{algorithmic}
\end{algorithm}

\begin{algorithm}[t]
	
	\caption{\textsc{Semi-Verified-BinarySearch}}
	\label{alg:robust_binary_search}
	\begin{algorithmic}[1]
		
		\REQUIRE {A sorted instance set $\hat{S} = \{\hat{x}_i\}_{i=1}^n$, $\alpha$, $\delta_1$.}
		\ENSURE {A labeled set $\bar{S}$.}

		\STATE $\tmin \assign 1$, $\tmax \assign n$, $k_2 \assign \frac{32}{(1-\eta)^2\alpha^{2}}\log\frac{6\log n}{\delta_1}$. 
		
		\WHILE{$\tmin < \tmax$}
		\STATE		$t \assign (\tmin + \tmax)/2$. Draw a set of $k_2$ workers $\WL$ from $\PL$.
		%
		
		\IF{$\alpha \geq 0.7$ or $\smaj_{\WL}({x}_t)\geq 1-\frac{(1-\eta)\alpha}{4}$}{ \label{step:maj_size_alpha}
			\STATE	$\hat{y}_t \assign \maj_{\WL}(x_t)$. \textbf{If} $\hat{y}_t=1$, \textbf{then} $\tmax \assign t-1$; \textbf{else} $\tmin \assign t+1$.
		}\ELSE{
			\STATE	Get a verified label $y_t^* \assign h^*(x_t)$. 
			{\bfseries Restart} Algoirthm \ref{alg:boost} with $\PL \assign \PL_{\mid ({x}_t,y_t^*)}$ and $\alpha \assign \frac{(1-\eta)\alpha}{1-\frac{(1-\eta)\alpha}{8}}$. 
		}
		\ENDIF
		
		\ENDWHILE
		\STATE	For all $t' > t, \hat{y}_{t'} \assign +1$. For all $t' < t, \hat{y}_{t'} \assign -1$.
		%
		%
		
		{\bfseries return} $\bar{S} = \{(\hat{x}_1, \hat{y}_1), (\hat{x}_2,\hat{y}_2), \dots, (\hat{x}_m, \hat{y}_m)\}$.
		
	\end{algorithmic}
\end{algorithm}

The comparison-based PAC learning algorithms of~\cite{xu2017noise,zeng2022crowd} aid label efficiency with a ``compare-and-label'' approach, which first sorts all the instances in a set $S$ and then query the label of the instances during binary search. We fortify such method by adding certain check points for both the ``compare'' and ``label'' steps to ensure that the majority vote is always correct (with high probability). We summarize the performance guarantees of Algorithms~\ref{alg:robust_quick_sort} and \ref{alg:robust_binary_search} below.

\begin{lemma}\label{lem:beta}
Consider Algorithm~\ref{alg:robust_quick_sort}. Given any set of instances $S$, with probability $1-\frac{2\delta_1}{3}$, the algorithm either increases the fraction of Massart workers from $\beta$ to  $\frac{(1-\eta)\beta}{1-\frac{(1-\eta)\beta}{8}}$, or correctly sorts $S$.
\end{lemma}

\begin{lemma}\label{lem:alpha}
Consider Algorithm~\ref{alg:robust_binary_search}. Assume $\hat{S}$ is correctly sorted. With probability $1-\frac{\delta_1}{3}$, the algorithm either increases the fraction of Massart workers  to  $\frac{(1-\eta)\alpha}{1-\frac{(1-\eta)\alpha}{8}}$, or correctly labels $\hat{S}$, namely, for all $(x,y)\in\bar{S}$, $y=h^*(x)$.
\end{lemma}

We further note that comparing to the non-robust algorithm in~\cite{zeng2022crowd},  \complabel does not require extra crowd labels or comparison tags to label $S$, because once it decides to query the verified data, it will end up restarting the main algorithm. Thus, the label and comparison complexity of \complabel matches that of~\citet{zeng2022crowd}.

\begin{proposition}\label{prop:complabel}
	Consider  Algorithm~\ref{alg:label}. If it does not restart and $\abs{S}\geq(\frac{1}{\delta_1})^{1/1000}$ , then with probability $1-\delta_1$, it correctly sorts and labels all the instances in $S$ with $O(\frac{\log\abs{S}}{\alpha^2}\cdot \log\log\abs{S})$ label queries and  $O(\frac{1}{\beta^2}\cdot \abs{S}\cdot\log^2\abs{S})$ comparison tags.
\end{proposition}

In light of Theorem~\ref{thm:boost} (i.e. boosting~\citep{schapire1990strength}), the two weak hypothese $h_1$ and $h_3$ are trained on distributions $\D$ and $\D_3$ which are not hard to construct if $h_2$ is given: $\D$ is the original marginal distribution, and $\D_3$ can be obtained by rejection sampling. By Assumption~\ref{as:pac}, drawing a set of $\Theta(n_{\sqrt\epsilon,\delta'})$ instances on each distribution respectively and feeding them to $\A_\H$ gives the following corollary.

\begin{corollary}\label{coro:h1}
	With probability $1-\frac{\delta'}{3}$, $\err_{\D}(h_1)\leq\frac{\sqrt\epsilon}{2}$. With probability $1-\frac{\delta'}{3}$, $\err_{\D_3}(h_3)\leq\frac{\sqrt\epsilon}{2}$.
\end{corollary}

\subsubsection{\filter}\label{subsec:filter}

\begin{algorithm}[t]
\caption{\filter}
\label{alg:filter}
\begin{algorithmic}[1]

\REQUIRE {Set of instances $S$, classifier $h$ with error  $\err_{\DX}(h)\in\big[\frac{\sqrt{\epsilon}}{6},\frac{\sqrt{\epsilon}}{2}\big]$, confidence $\delta_2$, parameters $\alpha,\beta$.}
\ENSURE {A set $S_I$ whose instances are misclassified by $h$.}

\STATE $b \assign \frac{4}{\sqrt{\epsilon}} \log\frac{16}{\delta_2} + (\frac{24}{\delta_2})^{1/1000}$.

\STATE Sample uniformly  $U \subset S$ of $b$ instances. $\bar{U} \assign \complabel(U,\alpha,\beta, \delta_2/4)$. 

\STATE   $x^- \gets$ the rightmost instance of those labeled as $-1$,  $x^+ \gets$  the leftmost instance of those labeled as $+1$.

\STATE $S_I \assign \emptyset$, $S_{\text{in}} \assign \emptyset$, $N \assign \frac{1}{\beta^2}\log\frac{1}{\epsilon}$.

\IF{$\beta < 0.7$}{
\STATE Run \test$(x^-, \beta, \delta_2/8)$, \test$(x^+, \beta, \delta_2/8)$.
}
\ENDIF

\FOR{$x\in S\backslash U$}{

\STATE $\text{ANS} \assign \text{YES}$.

\FOR{$t=1, \dots, N$}{

\STATE Draw a worker $t \sim \PC$ to obtain the comparison tag $Z_t(x,x^-)$. {\bfseries If} $Z_t(x,x^-) = \{ x < x^-\}$, {\bfseries then} $Z_t(x,x^+) \assign \{x < x^+\}$, {\bfseries else} query  $Z_t(x,x^+)$. \label{step:draw}

\STATE \textbf{If} $t$ is even, \textbf{then} {\bfseries continue} to the next iteration.

\STATE {Filtering: }\textbf{If} $\big[\maj(Z_{1:t}(x,x^-))= \{ x < x^- \}$ and $h(x)=-1\big]$ or $\big[\maj(Z_{1:t}(x,x^+))= \{ x > x^+ \}$ and $h(x)=1\big]$, \textbf{then} $\text{ANS} \assign \text{NO}$ and {\bfseries break}. \label{step:filter-outside}

}
\ENDFOR

\STATE \textbf{If} $\maj\big(Z_{1:N}(x,x^-)\big) = \{ x > x^- \}$ and $\maj\big(Z_{1:N}(x,x^+)\big) = \{ x < x^+ \}$ \textbf{then} $\text{ANS} \assign \text{NO}$ and $S_{\text{in}} \assign S_{\text{in}} \cup \{x\}$.\label{step:filter-inside}

\STATE \textbf{If} $\text{ANS} = \text{YES}$, \textbf{then} $S_I \assign S_I \cup \{x\}$.

}
\ENDFOR

\STATE $\bar{S_{\text{in}}} \assign \complabel(S_{\text{in}}, \delta_2/4,\alpha,\beta)$.

\STATE $S_I \assign S_I \cup \{x: (x, y) \in \bar{U} \cup \bar{S_{\text{in}}}\ \text{and}\ y \neq h(x) \}$.

{\bfseries return} $S_I$.

\end{algorithmic}
\end{algorithm}

A core component of Algorithm~\ref{alg:boost-comp} is a correct filtering process.
\cite{awasthi2017efficient} achieved this by executing a Phase~0 that examines the crowd labels for a randomly drawn instance set. However, this approach does not apply in the comparison-equipped setting because it relies on two support instances $x^-,x^+$ as reference to identify the informative instances. 
To be more concrete, the filtering scheme incorporates a similar idea to that of ``compare-and-label'' in order to achieve label efficiency.
Given a large enough sample set $S$, the goal is to identify two support instances $x^-,x^+$ that are close enough to the ``threshold'' where the labels in $S$ shift. By spending a few label queries to infer the labels of $x^-,x^+$, it help to filter a large number of instances in $S$ by querying pairwise comparisons. This is because $\forall x$, $[h^*(x^-)=-1,x<x^-]$ imply $h^*(x)=-1$ (similar holds for $x^+$). The pairwise comparisons of any $x$ with respect to $x^-,x^+$ become exceptionally informative. 

We hence build on the non-semi-verified filter of~\citet{zeng2022crowd} and integrate a new component, \test, that ensures to recover a distribution $\D'$ that is close enough to {$\D_2$} even in presence of the adversarial workers. In particular, we are concerned with two regions, $R_1$ and $R_2$ 
\begin{align}\label{eq:R1_R2}
R_1 := \Big\{x: \smaj_{\PC}(x, x^-) \leq 1-\frac{\beta}{2} \Big\},\notag\\ 
R_2 := \Big\{ x: \smaj_{\PC}(x, x^+) \leq  1-\frac{\beta}{2} \Big\},
\end{align}
which include any instance $x$ that would form a low-confidence pair with either support instance $x^-$ or $x^+$. Observe that both $x^-, x^+$ are fixed. If the {probability mass} of $R_1\cup R_2$ under distribution $\D$, denoted by $\D[R_1\cup R_2]$, is smaller than $\epsilon$, it means the filter uses $x^-,x^+$ to correctly identify the instances except for the $\epsilon$ fraction in $R_1\cup R_2$. Meanwhile, \test works on a smaller set of instances than what \filter is expected to work on. This is a key design to guarantee that the total comparison complexity does not blow up.

\begin{algorithm}[t]

	\caption{\textsc{Test}}
	\label{alg:test}
	
	\begin{algorithmic}[1]
		
		\REQUIRE {An instance $x'$, parameter $\beta$, confidence $\delta_2$.}

		\STATE Draw a set $S_r$ of $\frac{4}{\epsilon}\log\frac{8}{\delta_2}$ instances from $\D$.
		
		\STATE  $N\assign\frac{32}{((1-\eta)\beta)^2}\cdot\log(\frac{32\abs{S_r}}{\delta_2})$.
		
		\FOR{$x\in S_r$}{
			
			\STATE	Draw a set of $N$ workers $\W$ from $\PC$ and obtain the comparison tags on $(x,x')$.
			
			\IF{$\smaj_{\W}(x,x')< 1-\frac{(1-\eta)\beta}{4}$}{
				
				\STATE	Get a verified comparison tag $z^* \assign Z^*(x,x')$. 
				\STATE {\bfseries Restart} Algorithm \ref{alg:boost} with $\PC\assign \PC_{ \mid (x, x', z^*)}$ and $\beta \assign \frac{(1-\eta)\beta}{1-\frac{(1-\eta)\beta}{8}}$. 
				
			}
			\ENDIF
		}
		\ENDFOR
	\end{algorithmic}
	
\end{algorithm}

\begin{lemma}\label{lem:test}
	If Algorithm~\ref{alg:test} terminates without restarting Algorithm~\ref{alg:boost-comp}, we have the probability mass  $\D[R_1\cup R_2]\leq\frac{\epsilon}{4}$  with probability $1-\frac{\delta_2}{4}$. In addition, the comparison complexity is  $O\big(\frac{1}{\epsilon\beta^2}\log\big(\frac{1}{\delta_2}\big)\log\big(\frac{1}{\epsilon\delta_2}\big)\big)$.
\end{lemma}

\filter ensures that in Phase~2, the learner draws a set of $\Theta(n_{\sqrt\epsilon,\delta'})$ instances from a distribution $\D'$ that is a good simulation of $D_2$, which suffices to learn a good hypothesis $h_2$.


\begin{lemma}\label{lem:h2}
	Consider Phase~2 of Algorithm~\ref{alg:boost-comp}. With probability $1-\frac{\delta'}{3}$, $\err_{\D_2}(h_2)\leq\frac{\sqrt\epsilon}{2}$.
\end{lemma}

Now we have three hypotheses $h_1,h_2,h_3$ that satisfy the requirement for boosting; this suffices to produce a hypothesis $\hat{h}$ with the desired PAC guarantee. We further note that each pruning step helps to remove a significant fraction of the adversarial workers but only a small fraction of the Massart workers.
Hence, our main theorem, Theorem~\ref{thm:main_bothAlphaBeta}, follows. See Appendix~\ref{sec:proof} for the full proof.

\begin{theorem}[Comparison-equipped learning]\label{thm:main_bothAlphaBeta}
	Suppose that Assumptions~\ref{as:massart-label}, \ref{as:pac}, and \ref{as:massart-comp} are satisfied. Given any $\alpha, \beta, \epsilon, \delta \in (0, 1)$, assume $\eta < \frac{\min(\alpha, \beta)}{16}$. Denote $\Talpha = \log_{\frac{1}{1-\eta}} (\frac{\alpha-\alpha\eta}{\alpha - \eta})$ and $\Tbeta = \log_{\frac{1}{1-\eta}} (\frac{\beta-\beta\eta}{\beta - \eta})$. The following holds with probability $1-\delta$. There exists an algorithm (Algorithm~\ref{alg:boost-comp}) that runs in $O(\poly(d, \frac{1}{\epsilon}, {\Talpha}, {\Tbeta}, \frac{1}{\alpha}, \frac{1}{\beta}))$ time and learns $\H$ by returning a hypothesis $h: \X \rightarrow \Y$ with $\err_\DX(h)\leq\epsilon$. In addition, $m_L = \Talpha^2\big(\Talpha + \Tbeta\big) \cdot \tilde{O}\big( \log\frac{d + (\Talpha + \Tbeta)\frac{1}{\delta}}{\epsilon} \big)$, $m_C = \Tbeta^2 \cdot \tilde{O}_{\delta}\big( \big(\Talpha + \Tbeta\big)^{\frac{1001}{1000}} \cdot \frac{d}{\epsilon} \big)$, and $m_V = \Talpha + \Tbeta$. Therefore, $\labeloverhead = \tilde{o}\big( \Talpha^2(\Talpha + \Tbeta)\big)$ and $\compareoverhead = \tilde{O}_{\delta}\big( \Tbeta^2\big( \Talpha+ \Tbeta\big)^{\frac{1001}{1000}}\big)$ when $\epsilon\in\(0,(\log d)^{-4}\)$. In particular, when $\alpha$ and $\beta$ are constants, we have $\labeloverhead = o(1)$, $\compareoverhead = O(1)$, and $m_V = O(1)$.
\end{theorem}

\section{Conclusion}\label{sec:conclusion}


In this paper, we studied the problem of semi-verified learning threshold functions from the crowd. We showed that when the majority of the crowd workers are adversarial and the rest behave as Massart noise, it is still possible to achieve PAC guarantees without distributional assumptions. In addition, our algorithms enjoy both query and label-efficiency, and run in polynomial time.

\clearpage
\bibliography{../../jshen_ref,../../szeng_ref}

\newcommand{\etalchar}[1]{$^{#1}$}
\begin{thebibliography}{DKK{\etalchar{+}}20b}

\bibitem[AAKP20]{agarwal2020rank}
Arpit Agarwal, Shivani Agarwal, Sanjeev Khanna, and Prathamesh Patil.
\newblock Rank aggregation from pairwise comparisons in the presence of
  adversarial corruptions.
\newblock In {\em Proceedings of the 37th International Conference on Machine
  Learning}, pages 85--95, 2020.

\bibitem[AB99]{anthony1999neural}
Martin Anthony and Peter~L. Bartlett.
\newblock {\em Neural Network Learning: {T}heoretical Foundations}.
\newblock Cambridge University Press, 1999.

\bibitem[ABHM17]{awasthi2017efficient}
Pranjal Awasthi, Avrim Blum, Nika Haghtalab, and Yishay Mansour.
\newblock Efficient {PAC} learning from the crowd.
\newblock In {\em Proceedings of the 30th Annual Conference on Learning
  Theory}, pages 127--150, 2017.

\bibitem[ABL17]{awasthi2017power}
Pranjal Awasthi, Maria{-}Florina Balcan, and Philip~M. Long.
\newblock The power of localization for efficiently learning linear separators
  with noise.
\newblock {\em Journal of the {ACM}}, 63(6):50:1--50:27, 2017.

\bibitem[AL87]{angluin1987learning}
Dana Angluin and Philip~D. Laird.
\newblock Learning from noisy examples.
\newblock {\em Machine Learning}, 2(4):343--370, 1987.

\bibitem[BFKV96]{blum1996polynomial}
Avrim Blum, Alan~M. Frieze, Ravi Kannan, and Santosh~S. Vempala.
\newblock A polynomial-time algorithm for learning noisy linear threshold
  functions.
\newblock In {\em Proceedings of the 37th Annual {IEEE} Symposium on
  Foundations of Computer Science}, pages 330--338, 1996.

\bibitem[BK20]{bakshi2020list}
Ainesh Bakshi and Pravesh Kothari.
\newblock List-decodable subspace recovery via sum-of-squares.
\newblock {\em CoRR}, abs/2002.05139, 2020.

\bibitem[CSV17]{charika2017learning}
Moses Charikar, Jacob Steinhardt, and Gregory Valiant.
\newblock Learning from untrusted data.
\newblock In {\em Proceedings of the 49th Annual {ACM} {SIGACT} Symposium on
  Theory of Computing}, pages 47--60, 2017.

\bibitem[DGT19]{diakonikolas2019distribution}
Ilias Diakonikolas, Themis Gouleakis, and Christos Tzamos.
\newblock Distribution-independent {PAC} learning of halfspaces with {Massart}
  noise.
\newblock In {\em Proceedings of the 33rd Annual Conference on Neural
  Information Processing Systems}, pages 4751--4762, 2019.

\bibitem[DK22]{diamassarthard}
Ilias Diakonikolas and Daniel Kane.
\newblock Near-optimal statistical query hardness of learning halfspaces with
  massart noise.
\newblock In {\em Conference on Learning Theory}, pages 4258--4282, 2022.

\bibitem[DKK20a]{dia2020list}
Ilias Diakonikolas, Daniel Kane, and Daniel Kongsgaard.
\newblock List-decodable mean estimation via iterative multi-filtering.
\newblock In {\em Proceedings of the 34th Annual Conference on Neural
  Information Processing Systems}, 2020.

\bibitem[DKK{\etalchar{+}}20b]{diakonikolas2020polynomial}
Ilias Diakonikolas, Daniel~M. Kane, Vasilis Kontonis, Christos Tzamos, and
  Nikos Zarifis.
\newblock A polynomial time algorithm for learning halfspaces with {Tsybakov}
  noise.
\newblock {\em CoRR}, abs/2010.01705, 2020.

\bibitem[DKK{\etalchar{+}}21]{dia2021agnostic}
Ilias Diakonikolas, Daniel~M. Kane, Vasilis Kontonis, Christos Tzamos, and
  Nikos Zarifis.
\newblock Agnostic proper learning of halfspaces under gaussian marginals.
\newblock In {\em Proceedings of the 34th Annual Conference on Learning
  Theory}, pages 1522--1551, 2021.

\bibitem[DKS18]{dia2018list}
Ilias Diakonikolas, Daniel~M. Kane, and Alistair Stewart.
\newblock List-decodable robust mean estimation and learning mixtures of
  spherical gaussians.
\newblock In {\em Proceedings of the 50th Annual {ACM} {SIGACT} Symposium on
  Theory of Computing}, pages 1047--1060. {ACM}, 2018.

\bibitem[DKTZ20]{diakonikolas2020learning}
Ilias Diakonikolas, Vasilis Kontonis, Christos Tzamos, and Nikos Zarifis.
\newblock Learning halfspaces with {Massart} noise under structured
  distributions.
\newblock In {\em Proceedings of the 33rd Annual Conference on Learning
  Theory}, pages 1486--1513, 2020.

\bibitem[DS09]{dekel2009vox}
Ofer Dekel and Ohad Shamir.
\newblock Vox populi: Collecting high-quality labels from a crowd.
\newblock In {\em Proceedings of the 22nd Conference on Learning Theory}, 2009.

\bibitem[FH10]{furnkranz2010prefer}
Johannes F{\"{u}}rnkranz and Eyke H{\"{u}}llermeier.
\newblock Preference learning and ranking by pairwise comparison.
\newblock In {\em Preference Learning}, pages 65--82. Springer, 2010.

\bibitem[FHO{\etalchar{+}}17]{falahatgar2017max}
Moein Falahatgar, Yi~Hao, Alon Orlitsky, Venkatadheeraj Pichapati, and Vaishakh
  Ravindrakumar.
\newblock Maxing and ranking with few assumptions.
\newblock In {\em Proceedings of the 31st Annual Conference on Neural
  Information Processing Systems}, pages 7060--7070, 2017.

\bibitem[FJO{\etalchar{+}}18]{falahatgar2018limit}
Moein Falahatgar, Ayush Jain, Alon Orlitsky, Venkatadheeraj Pichapati, and
  Vaishakh Ravindrakumar.
\newblock The limits of maxing, ranking, and preference learning.
\newblock In {\em Proceedings of the 35th International Conference on Machine
  Learning}, pages 1426--1435, 2018.

\bibitem[GR06]{guruswami2006hardness}
Venkatesan Guruswami and Prasad Raghavendra.
\newblock Hardness of learning halfspaces with noise.
\newblock In {\em Proceedings of the 47th Annual {IEEE} Symposium on
  Foundations of Computer Science}, pages 543--552, 2006.

\bibitem[Hau92]{haussler1992decision}
David Haussler.
\newblock Decision theoretic generalizations of the {PAC} model for neural net
  and other learning applications.
\newblock {\em Information and Computation}, 100(1):78--150, 1992.

\bibitem[HKLM20]{hopkins2020noise}
Max Hopkins, Daniel Kane, Shachar Lovett, and Gaurav Mahajan.
\newblock Noise-tolerant, reliable active classification with comparison
  queries.
\newblock In {\em Proceedings of the 33rd Annual Conference on Learning
  Theory}, pages 1957--2006, 2020.

\bibitem[JN11]{jamieson2011active}
Kevin~G. Jamieson and Robert~D. Nowak.
\newblock Active ranking using pairwise comparisons.
\newblock In {\em Proceedings of the 25th Annual Conference on Neural
  Information Processing Systems}, pages 2240--2248, 2011.

\bibitem[KKK19]{karmalkar2019list}
Sushrut Karmalkar, Adam~R. Klivans, and Pravesh Kothari.
\newblock List-decodable linear regression.
\newblock In {\em Proceedings of the 33rd Annual Conference on Neural
  Information Processing Systems}, pages 7423--7432, 2019.

\bibitem[KKMS08]{kalai2008agnostically}
Adam~Tauman Kalai, Adam~R. Klivans, Yishay Mansour, and Rocco~A. Servedio.
\newblock Agnostically learning halfspaces.
\newblock {\em {SIAM} Journal on Computing}, 37(6):1777--1805, 2008.

\bibitem[KLMZ17]{kane2017active}
Daniel~M. Kane, Shachar Lovett, Shay Moran, and Jiapeng Zhang.
\newblock Active classification with comparison queries.
\newblock In {\em Proceedings of the 58th Annual {IEEE} Symposium on
  Foundations of Computer Science}, pages 355--366, 2017.

\bibitem[KS17]{kothari2017better}
Pravesh~K. Kothari and Jacob Steinhardt.
\newblock Better agnostic clustering via relaxed tensor norms.
\newblock {\em CoRR}, abs/1711.07465, 2017.

\bibitem[KSS92]{kearns1992toward}
Michael~J. Kearns, Robert~E. Schapire, and Linda Sellie.
\newblock Toward efficient agnostic learning.
\newblock In {\em Proceedings of the 5th Annual Conference on Computational
  Learning Theory}, pages 341--352, 1992.

\bibitem[KV94]{kearns1994intro}
Michael~J. Kearns and Umesh~V. Vazirani.
\newblock {\em An Introduction to Computational Learning Theory}.
\newblock {MIT} Press, 1994.

\bibitem[MN06]{massart2006risk}
Pascal Massart and {\'E}lodie N{\'e}d{\'e}lec.
\newblock Risk bounds for statistical learning.
\newblock {\em The Annals of Statistics}, pages 2326--2366, 2006.

\bibitem[MV18]{meister2018data}
Michela Meister and Gregory Valiant.
\newblock A data prism: {S}emi-verified learning in the small-alpha regime.
\newblock In {\em Proceedings of the 31st Conference On Learning Theory}, pages
  1530--1546, 2018.

\bibitem[PNZ{\etalchar{+}}15]{park2015preference}
Dohyung Park, Joe Neeman, Jin Zhang, Sujay Sanghavi, and Inderjit~S. Dhillon.
\newblock Preference completion: Large-scale collaborative ranking from
  pairwise comparisons.
\newblock In {\em Proceedings of the 32nd International Conference on Machine
  Learning}, pages 1907--1916, 2015.

\bibitem[PSM17]{prelec2017solution}
Dra{\v{z}}en Prelec, H~Sebastian Seung, and John McCoy.
\newblock A solution to the single-question crowd wisdom problem.
\newblock {\em Nature}, 541(7638):532--535, 2017.

\bibitem[RLS20]{ren2020sample}
Wenbo Ren, Jia Liu, and Ness~B. Shroff.
\newblock The sample complexity of best-k items selection from pairwise
  comparisons.
\newblock {\em CoRR}, abs/2007.03133, 2020.

\bibitem[Ros58]{rosenblatt1958perceptron}
Frank Rosenblatt.
\newblock The {P}erceptron: {A} probabilistic model for information storage and
  organization in the brain.
\newblock {\em Psychological review}, 65(6):386--408, 1958.

\bibitem[RY20a]{raghavendra2020list}
Prasad Raghavendra and Morris Yau.
\newblock List decodable learning via sum of squares.
\newblock In {\em Proceedings of the 2020 {ACM-SIAM} Symposium on Discrete
  Algorithms}, pages 161--180, 2020.

\bibitem[RY20b]{raghavendra2020subspace}
Prasad Raghavendra and Morris Yau.
\newblock List decodable subspace recovery.
\newblock In {\em Proceedings of the 33rd Annual Conference on Learning
  Theory}, pages 3206--3226, 2020.

\bibitem[SBGW16]{shah2016stochastic}
Nihar~B. Shah, Sivaraman Balakrishnan, Aditya Guntuboyina, and Martin~J.
  Wainwright.
\newblock Stochastically transitive models for pairwise comparisons:
  Statistical and computational issues.
\newblock In {\em Proceedings of the 33nd International Conference on Machine
  Learning}, pages 11--20, 2016.

\bibitem[SBW19]{shah2019feel}
Nihar~B. Shah, Sivaraman Balakrishnan, and Martin~J. Wainwright.
\newblock Feeling the bern: Adaptive estimators for bernoulli probabilities of
  pairwise comparisons.
\newblock {\em {IEEE} Transactions on Information Theory}, 65(8):4854--4874,
  2019.

\bibitem[Sch90]{schapire1990strength}
Robert~E. Schapire.
\newblock The strength of weak learnability.
\newblock {\em Machine Learning}, 5:197--227, 1990.

\bibitem[She21]{shen2021power}
Jie Shen.
\newblock On the power of localized {P}erceptron for label-optimal learning of
  halfspaces with adversarial noise.
\newblock In {\em Proceedings of the 38th International Conference on Machine
  Learning}, pages 9503--9514, 2021.

\bibitem[Slo88]{sloan1988types}
Robert~H. Sloan.
\newblock Types of noise in data for concept learning.
\newblock In {\em Proceedings of the First Annual Workshop on Computational
  Learning Theory}, pages 91--96, 1988.

\bibitem[SVC16]{steinhardt2016avoid}
Jacob Steinhardt, Gregory Valiant, and Moses Charikar.
\newblock Avoiding imposters and delinquents: Adversarial crowdsourcing and
  peer prediction.
\newblock In {\em Proceedings of the 30th Annual Conference on Neural
  Information Processing Systems}, pages 4439--4447, 2016.

\bibitem[SW17]{shah2017simple}
Nihar~B. Shah and Martin~J. Wainwright.
\newblock Simple, robust and optimal ranking from pairwise comparisons.
\newblock {\em Journal of Machine Learning Research}, 18:199:1--199:38, 2017.

\bibitem[Tsy04]{tsybakov2004optimal}
Alexander~B. Tsybakov.
\newblock Optimal aggregation of classifiers in statistical learning.
\newblock {\em The Annals of Statistics}, 32(1):135--166, 2004.

\bibitem[Val84]{valiant1984theory}
Leslie~G. Valiant.
\newblock A theory of the learnable.
\newblock {\em Communications of the {ACM}}, 27(11):1134--1142, 1984.

\bibitem[Vau17]{vaughan2017making}
Jennifer~Wortman Vaughan.
\newblock Making better use of the crowd: How crowdsourcing can advance machine
  learning research.
\newblock {\em Journal of Machine Learning Research}, 18:193:1--193:46, 2017.

\bibitem[XZS{\etalchar{+}}17]{xu2017noise}
Yichong Xu, Hongyang Zhang, Aarti Singh, Artur Dubrawski, and Kyle Miller.
\newblock Noise-tolerant interactive learning using pairwise comparisons.
\newblock In {\em Proceedings of the 31st Annual Conference on Neural
  Information Processing Systems}, pages 2431--2440, 2017.

\bibitem[ZS22a]{zeng2022crowd}
Shiwei Zeng and Jie Shen.
\newblock Efficient {PAC} learning from the crowd with pairwise comparisons.
\newblock In {\em Proceedings of the 39th International Conference on Machine
  Learning}, pages 25973--25993, 2022.

\bibitem[ZS22b]{zsLDsparse2022}
Shiwei Zeng and Jie Shen.
\newblock List-decodable sparse mean estimation.
\newblock {\em CoRR}, abs/2205.14337, 2022.

\bibitem[ZSA20]{zhang2020efficient}
Chicheng Zhang, Jie Shen, and Pranjal Awasthi.
\newblock Efficient active learning of sparse halfspaces with arbitrary bounded
  noise.
\newblock In {\em Proceedings of the 34th Annual Conference on Neural
  Information Processing Systems}, pages 7184--7197, 2020.

\end{thebibliography}
\bibliographystyle{alpha}

\appendix

\onecolumn

\section{Omitted Proof of Theorem~\ref{thm:agnostic}}

\begin{proof}
Recall that in the agnostic model, the adversary can flip an arbitrary $(1-\alpha)$-fraction of the labels in an adversarial manner while retaining the marginal distribution on $\X$. It is shown in \citet{kalai2008agnostically,dia2021agnostic} that if $\H_{\mathrm{hs}}$ is the class of homogeneous halfspaces and $D$ is the standard Gaussian, then there exists a learning algorithm which takes as input $\theta(d^{\poly(1/\epsilon)})$ samples generated in such a way, runs in time $O((d/\epsilon)^{\poly(1/\epsilon)})$, and returns a hypothesis with error rate less than $(1-\alpha) + \epsilon$ with overwhelming probability.

Now we note that in our crowdsourcing model, each instance is assigned one worker which is randomly chosen from the pool and may be adversarial with probability $1-\alpha$, while otherwise he is perfect. Therefore, the labels gathered in such a way satisfy the condition of agnostic noise model, and we apply the above results directly which gives Theorem~\ref{thm:agnostic}.
\end{proof}

\section{Omitted Proof of Theorem~\ref{thm:main_labelonly}}\label{sec:app:proof-label}

The key idea of the proof is the following: we show  that with the carefully designed pruning approach, the adversarial workers either corrupt the labels, under which a noticeable fraction of them will be pruned away and Algorithm~\ref{alg:boost} will be restarted with the cleaner pool of workers, or they provide  labels in such a way that the majority vote is correct and Algorithm~\ref{alg:boost} is exactly mirroring the easier case where the majority is correct. Our goal here is to show that in the former case, the algorithm must make significant progress such that it obtains a pool of workers with most of them, say $70\%$, are Massart workers, for which we show that with high probability, a constant labeling overhead suffices to guarantee PAC learnability. Therefore, at the technical level, our algorithm and analysis are different from \cite{awasthi2017efficient} in two aspects: first, we draw new analysis to handle Massart workers when bounding the restarting time, and second, we show that a majority of Massart workers also suffices under the condition that the Massart noise rate $\eta$ is not large.

\begin{lemma}\label{lem:alpha-labelonly}
Consider the \textsc{Prune-and-Label} subroutine in Algorithm~\ref{alg:boost}. With probability $1-\delta_1$, the algorithm either increases the fraction of Massart workers from $\alpha$ to  $\frac{(1-\eta)\alpha}{1-\frac{(1-\eta)\alpha}{8}}$, or correctly labels $\hat{S}$, namely, for all $(x,y)\in\bar{S}$, $y=h^*(x)$.
\end{lemma}
\begin{proof}
Recall that Algorithm~\ref{alg:prune-label} queries a set $\WL$ of $k_0=\frac{1}{2(1-\eta)^2\alpha^{2}}\cdot\log\frac{1000m}{\delta_1}$ workers for labeling each instance.
Since $1-\eta\in(\frac12,1]$, by applying the Chernoff bound, we have
\begin{equation*}
\Pr\Big(\abs{\smaj_{\PL}(x) - \smaj_{\WL}(x)} \geq\frac{(1-\eta)\alpha}{8} \Big) \leq 2\cdot e^{-2k_0\cdot\(\frac{\alpha}{16}\)^2} \leq \frac{\delta_1}{1000m}.
\end{equation*}
Taking the union bound over the labeling of $\log n$ instances, we have that with probability at least $1-\delta_1$, for all $x \in S$, $\abs{\smaj_{\PL}(x) - \smaj_{\WL}(x)} \leq \frac{(1-\eta)\alpha}{8}$. Thus, if $\smaj_{\WL}(x) \geq 1 - \frac{(1-\eta)\alpha}{4}$, we have the population majority size $\smaj_{\PL}(x) \geq 1 - \frac{3(1-\eta)\alpha}{8} > 1 - (1-\eta)\alpha$ with high probability, indicating that the majority voting is correct, i.e. $\maj_{\WL}(x)=h^*(x)$. 
As a result, if the algorithm finishes without restarting, we are sure that $\forall (x,y)\in\bar{S}, y=h^*(x)$. On the flip side, if the algorithm prunes due to $\smaj_{\WL}(x) < 1 - \frac{(1-\eta)\alpha}{4}$, we have $\smaj_{\PL}(x) < 1 - \frac{(1-\eta)\alpha}{8}$ with high probability, indicating that if we query a verified label for $x$, it is guaranteed that more than a $\frac{(1-\eta)\alpha}{8}$-fraction of the workers can be pruned (among which only $\eta\alpha$ could be Massart workers), and $\alpha \assign \frac{(1-\eta)\alpha}{1-\frac{(1-\eta)\alpha}{8}}$. 
\end{proof}

\begin{proposition}\label{prop:prune-label}
Consider Algorithm~\ref{alg:prune-label}. If it does not restart, then with probability $1-\delta_1$, it correctly labels all the instances in $S$ with $O(\frac{m\log m}{\alpha^2})$ label queries.
\end{proposition}
\begin{proof}
Due to Lemma~\ref{lem:alpha-labelonly}, for any input set $S$ of size $m$, the total number of crowd labels is
$k_0\cdot O(\abs{S}) = O\(\frac{1}{\alpha^2}\log\(\frac{m}{\delta_1}\)\cdot m \) 
= O\(\frac{m\log m}{\alpha^2}\).$
\end{proof}

\begin{lemma}\label{lem:restart_labelonly}
Consider Algorithm~\ref{alg:boost}. $\forall\alpha\in(0,1]$, if $\eta<\frac{\alpha}{16}$, then with probability $1-\delta$, Algorithm~\ref{alg:boost} will restart $O(\Talpha)$ times. In addition, the required number of verified labels from $\expert$ is $m_V=O(\Talpha)$.
\end{lemma}
\begin{proof}
By Lemma~\ref{lem:alpha-labelonly}, each time we get a verified label $y^*$ from $\expert$, we prune a $\frac{(1-\eta)\alpha}{8}$ fraction of the workers with high probability. Since it is not guaranteed that all Massart workers would give the correct labels, some of them would be pruned. Given that the pool is large, with probability $1-\delta/\tau$,  at most an $\eta$ fraction of the Massart workers will be pruned.
Denote by $\alpha_i$ the fraction of Massart workers and $N_i$ the total number of workers in the pool after $i$ prunings. Clearly, we have $\alpha_0=\alpha$. Note that 
\begin{equation*}
\alpha_{i} N_{i} \geq \alpha_{i-1}N_{i-1} \cdot (1-\eta) \geq \dots \geq \alpha_{0}N_{0}\cdot (1-\eta)^{i}.
\end{equation*}
Then, we have 
\begin{align}
\alpha_K &\geq \frac{\alpha_0 N_0\cdot\(1-\eta\)^{K}}{N_0 - \frac{(1-\eta)\alpha_0}{8}\cdot N_0 - \frac{(1-\eta)\alpha_1}{8}\cdot N_1 - \dots - \frac{(1-\eta)\alpha_{K-1}}{8}\cdot N_{K-1}} \notag\\
&= \frac{\alpha_0 N_0\cdot\(1-\eta\)^{K}}{N_0 - \frac{(1-\eta)\alpha_0N_0}{8}\cdot \(1 + (1-\eta) + \dots + (1-\eta)^K	\)} \notag\\
&= \frac{\alpha_0 N_0\cdot\(1-\eta\)^{K}}{N_0 - \frac{(1-\eta)\alpha_0N_0}{8}\cdot \frac{1-(1-\eta)^{K+1}}{\eta} } \notag\\
&= \frac{8\eta \cdot \alpha_0 \cdot\(1-\eta\)^{K}}{8\eta - \alpha_0(1-\eta) + \alpha_0(1-\eta)^{K+2} } \label{eq:alpha_K}
\end{align}

Consider parameters $\alpha_i$ and $\eta$. Obviously, it must be satisfied that 
\begin{equation*}
\eta < \frac{(1-\eta)\alpha_i}{8}\ \forall i\geq0,
\end{equation*}
as otherwise, more fraction of Massart workers would be removed than that of the adversarial workers, i.e. the fraction of Massart workers in the pool decreases. Since $\alpha_i\geq\alpha_0$ for any $i\geq0$, and $1-\eta \in (\frac12,1]$, it only requires $\eta < \frac{\alpha_0}{16}$.

From Eq.~\eqref{eq:alpha_K}, to pick a sufficient number of verified labels $K$ that increases $\alpha_K$ to any given constant $C\in[0.7,1]$, it suffices to choose
\begin{equation*}
K \geq \log_{(1-\eta)} \frac{C\alpha_0(1-\eta) - 8C\eta}{C\alpha_0(1-\eta)^2 - 8\eta\alpha_0} .
\end{equation*}

Due to $C\in[0.7,1]$, $1-\eta \in (\frac12,1]$, $\alpha_0<1$, it requires at most
\begin{equation*}
O\( \log_{\frac{1}{1-\eta}}\frac{\alpha_0 - \eta\alpha_0}{\alpha_0 - \eta} \)
= O\(T(\alpha_0,\eta)\)
\end{equation*}
verified label from the trusted oracle $\expert$ for $\alpha_K$ to surpass $0.7$.
By union bound, with probability at least $1-\delta$, it requires at most $O(\Talpha)$ verified labels to increase the fraction of Massart workers such that they form a strong majority, where the algorithm no longer prunes and must return the desired hypothesis. Therefore, the main algorithm only restarts for $O(\Talpha )$ times.
\end{proof}

\begin{proof}[Proof of Theorem~\ref{thm:main_labelonly}]
	Given that in all three phases, Algorithm~\ref{alg:boost} gathers a sample of size $O(m_{\sqrt{\epsilon},\delta'})$ and is correctly labeled with high probability. In addition, the algorithm is restarted at most $O(\Talpha)$ times (Lemma~\ref{lem:restart_labelonly}). Then, by union bound, Assumption~\ref{as:pac}, Lemma~\ref{lem:super} and Theorem~\ref{thm:boost}, we conclude that with probability at least $1-\delta$, Algorithm~\ref{alg:boost} returns $h$ such that $\err_{D}(h)\leq\epsilon$.
	
	The number of required verified labels from $\oracletrust$ is $m_V=O(\Talpha)$ by Lemma~\ref{lem:restart_labelonly}. It remains to show the label complexity and label overhead. From Lemma 4.9 of~\cite{awasthi2017efficient}and Proposition~\ref{prop:prune-label}, it requires
	\begin{equation}
	m_L =  O\( \Talpha \cdot m_{\epsilon,\delta'}  + \Talpha\cdot\frac{1}{\epsilon}\cdot\log\frac{1}{\delta'}\cdot\log\frac{1}{\epsilon\delta'} +  \Talpha \cdot \frac{m_{\sqrt{\epsilon},\delta'}\log m_{\sqrt{\epsilon},\delta'}}{\alpha^2} \)
	\end{equation}
label queries from the crowd, where we recall that $\delta' = \frac{\delta}{8 T_{\alpha, \eta}}$ as defined in Algorithm~\ref{alg:boost}. This combined with the definition of $m_{\epsilon, \delta}$ (see Eq.~\eqref{eq:m-eps-delta}) gives the announced label complexity of $\tilde{O}(\frac{d \Talpha}{\epsilon \alpha^2} \cdot \log\frac{1}{\delta})$.

Finally, the calculation of labeling overhead follows from the definition, i.e.
\begin{align*}
\labeloverhead = \frac{m_L}{m_{\epsilon, \delta}} 
&= O\( \Talpha \log\Talpha + \Talpha\log^2\Talpha + \Talpha \cdot \frac{m_{\sqrt{\epsilon},\delta'}\log m_{\sqrt{\epsilon},\delta'}}{\alpha^2 \cdot m_{\epsilon,\delta}} \) \\
&= O\( \Talpha \log^2\Talpha + \frac{\Talpha}{\alpha^2} \sqrt{\epsilon} \log d \cdot \log^2 \Talpha \).
\end{align*}
Thus, when $\epsilon < \log^{-2}d$, the labeling overhead reads as $O(\frac{\Talpha}{\alpha^2}  \log^2 \Talpha)$, which is upper bounded by a constant as far as $\alpha$ is a small constant.
The proof is complete.
\end{proof}

\section{Omitted Proof of Theorem~\ref{thm:main_bothAlphaBeta}} \label{sec:proof}

\begin{lemma}[Restatement of Lemma~\ref{lem:beta}]\label{lem:beta-restate}
Consider Algorithm~\ref{alg:robust_quick_sort}. Given any set of instances $S$, with probability $1-\frac{2\delta_1}{3}$, the algorithm either increases the fraction of Massart workers from $\beta$ to  $\frac{(1-\eta)\beta}{1-\frac{(1-\eta)\beta}{8}}$, or correctly sorts $S$.
\end{lemma}
\begin{proof}
	Given that $k_1 = \frac{1}{2\beta^{2}}\cdot\log\frac{3006n\cdot\log n }{\delta_1}$ in Algorithm~\ref{alg:robust_quick_sort}. By the guarantee of celebrated algorithm \quicksort, with probability $1-\frac{1}{n^c}$, the total number of pairs that needs to be compared is $(c+2)n\log n$. By the sample size of input set $S$ and setting $c=1000$,  with probability $1-\frac{\delta_1}{3}$, $(c+2)n\log n=1002n\log n$. In addition, similar to the proof of~\ref{lem:alpha}, we have $\smaj_{\PC}(x,x')=\E[\smaj_{\WC}(x,x')]$ for any $(x,x')$, and 
	\begin{equation*}
	\Pr\[\bigcup_{l=1}^{1002n\log n} \[\abs{\smaj_{\PC}(x,x')_l - \smaj_{\WC}(x,x')_l} \leq \frac{(1-\eta)\beta}{8} \]\] \leq \frac{\delta_1}{3}.
	\end{equation*}
	In other words, with probability $1-\frac{\delta_1}{3}$, the following is guaranteed. If $\smaj_{\WC}(x,x') \geq 1 - \frac{(1-\eta)\beta}{4}$, we have the population majority size $\smaj_{\PC}(x,x') \geq 1 - \frac{3(1-\eta)\beta}{8} > 1 - (1-\eta)\beta$, indicating that the majority voting is correct, i.e. $\maj_{\WC}(x,x') =  Z^*(x,x')$. 
	As a result, if the algorithm never restarts, $\hat{S}$ is correctly sorted according to $Z^*$. On the other hand, if the algorithm prunes because $\smaj_{\WC}(x,x') < 1 - \frac{(1-\eta)\beta}{4}$, we have $\smaj_{\PC}(x,x') \geq 1 - \frac{(1-\eta)\beta}{8}$ with high probability, meaning that by querying a verified comparison for $(x,x')$ we can remove at least a $\frac{(1-\eta)\beta}{8}$-fraction of the workers (with at most $\eta\beta$ being Massart) and update $\beta \assign \frac{(1-\eta)\beta}{1-\frac{(1-\eta)\beta}{8}}$. 
\end{proof}

\begin{lemma}[Restatement of Lemma~\ref{lem:alpha}]\label{lem:alpha-restate}
	Consider Algorithm~\ref{alg:robust_binary_search}. Assume $\hat{S}$ is correctly sorted. With probability $1-\frac{\delta_1}{3}$, the algorithm either increases the fraction of Massart workers  to  $\frac{(1-\eta)\alpha}{1-\frac{(1-\eta)\alpha}{8}}$, or correctly labels $\hat{S}$, namely, for all $(x,y)\in\bar{S}$, $y=h^*(x)$.
\end{lemma}
\begin{proof}
	{The proof follows the same pipeline as that of Lemma~\ref{lem:alpha-labelonly}, with parameter $k_2$ for \binarysearch algorithm and the fact that binary search only queries labels on at most $\log n$ instances.}
\end{proof}

\begin{proposition}[Restatement of Proposition~\ref{prop:complabel}]\label{prop:complabel-restate}
	Consider Algorithm~\ref{alg:label}. If it does not restart and $\abs{S}\geq(\frac{1}{\delta_1})^{1/1000}$ , then with probability $1-\delta_1$, it correctly sorts and labels all the instances in $S$ with $O(\frac{\log n}{\alpha^2}\cdot \log\log n)$ label queries and  $O(\frac{1}{\beta^2}\cdot n\log^2 n)$ comparison tags.
\end{proposition}
\begin{proof}
	Due to Lemma~\ref{lem:beta-restate} and \ref{lem:alpha-restate}, for any input set $S$ of size $n$, the total number of crowd labels is
	\begin{equation*}
		k_2\cdot O(\log\abs{S}) = O\(\frac{1}{\alpha^2}\log\(\frac{\log n}{\delta_1}\)\cdot\log n \) 
		= O\(\frac{\log n}{\alpha^2}\cdot \log\log n \),
	\end{equation*}
	and the total number of crowd comparison tags is
	\begin{equation*}
		k_1\cdot O\(\abs{S}\log\abs{S}\) = O\( \frac{1}{\beta^{2}}\cdot\log\frac{n\cdot\log n }{\delta_1} \cdot n\log n \) 
		= O\(\frac{1}{\beta^2}\cdot n\log^2 n\).
	\end{equation*}
\end{proof}

\begin{corollary}[Restatement of Corollary~\ref{coro:h1}]\label{coro:h1-restate}
	With probability $1-\frac{\delta'}{3}$, $\err_{\D}(h_1)\leq\frac{\sqrt\epsilon}{2}$. With probability $1-\frac{\delta'}{3}$, $\err_{\D_3}(h_3)\leq\frac{\sqrt\epsilon}{2}$.
\end{corollary}
\begin{proof}
	By drawing a sample $S_1$ of $n_{\sqrt{\epsilon}/2,\delta'/6}$ from $\D$ and labeling it by \complabel, with probability at least $1-\frac{\delta'}{6}$, $\bar{S_1}$ is labeled correctly according to $h^*$ (Proposition~\ref{prop:complabel-restate}). Furthermore, we note that $n_{\sqrt{\epsilon}/2,\delta'/6} \geq m_{\sqrt{\epsilon}/2,\delta'/6}$. By Assumption~\ref{as:pac}, $\err_{\D}(h_1)\leq\frac{\sqrt\epsilon}{2}$ with probability $1-\frac{\delta'}{3}$.
	
	Similarly, in Phase~3 of Algorithm~\ref{alg:boost-comp}, by rejection sampling we can successfully sample a set $S_3$ of $n_{\sqrt{\epsilon}/2,\delta'/6}$ from $\D_3$. Again by Proposition~\ref{prop:complabel} and Assumption~\ref{as:pac}, $\err_{\D_3}(h_3)\leq\frac{\sqrt\epsilon}{2}$ with probability $1-\frac{\delta'}{3}$.
\end{proof}

\begin{lemma}
	\label{lem:restart_restate}
	$\forall\alpha,\beta\in(0,1]$, with probability $1-\delta$, Algorithm~\ref{alg:boost-comp} will restart $O\big(\Talpha+\Tbeta\big)$ times. In addition, the required number of verified labels from $\expert$ is $O(\Talpha)$ and that of the verified comparison tags is $O(\Tbeta)$.
\end{lemma}
\begin{proof}
The deduction for removing adversarial workers from the pool of workers who provide comparison tags is similar to the one who provide labels. By union bound, with probability at least $1-\delta$, it requires at most $O(\Talpha)$ verified labels and $O(\Tbeta)$ verified comparison tags to increase the fraction of Massart workers such that they form a strong majority, where the algorithm no longer prunes and must return the desired hypothesis. Therefore, the main algorithm only restarts for $O(\Talpha +\Tbeta)$ times.
\end{proof}

\subsection{Performance guarantee of Phase~2}


\begin{lemma}[Restatement of Lemma~\ref{lem:test}]\label{lem:test-restate}
	If Algorithm~\ref{alg:test} terminates without restarting Algorithm~\ref{alg:boost-comp}, we have the probability mass  $\D[R_1\cup R_2]\leq\frac{\epsilon}{4}$  with probability $1-\frac{\delta_2}{4}$. In addition, the comparison complexity is  $O\big(\frac{1}{\epsilon\beta^2}\log\big(\frac{1}{\delta_2}\big)\log\big(\frac{1}{\epsilon\delta_2}\big)\big)$.
\end{lemma}
\begin{proof}
	Recall that by definition, $R_1 := \{x: \smaj_{\PC}(x, x^-) \leq 1-\frac{\beta}{2} \}$ and $R_2 := \{ x: \smaj_{\PC}(x, x^+) \leq  1-\frac{\beta}{2} \}$.
	Without loss of generality, we prove the lemma for $x^-$. We remark that the following guarantee holds for $x^+$ as well.
	
	For any instance $x\in S_r$, let 
	\begin{equation}
	\ell_i=
	\begin{cases}
	1, & \text{if}\ \text{worker $i$ agrees with $\maj_{\PC}(x,x^-)$,} \\
	0, & \text{otherwise.}
	\end{cases}
	\end{equation}
	Since we query a set $\W$ of $N =\frac{32}{\beta^2}\cdot\log\big(\frac{32\abs{S_r}}{\delta_2}\big)$ workers from the crowd to compare $(x,x^-)$. Note that $\smaj_{\PC}(x,x^-) = \E[\smaj_{\W}(x,x^-)] = \mu$, which is also the probability that $\ell_i=1$. Therefore, by Hoeffding's inequality we have
	\begin{align}
		\Pr\[\smaj_{\W}(x,x^-)\geq1-\frac{(1-\eta)\beta}{4}\] 
		&\leq \Pr\[\Big\lvert \smaj_{\W}(x,x^-)-\smaj_{\PC}(x,x^-)\Big\rvert\geq\frac{(1-\eta)\beta}{8} \] \notag\\
		&= \Pr\[\abs{\frac{1}{N}\sum_{i=1}^{N}\ell_i-\mu}\geq\frac{(1-\eta)\beta}{8} \] \notag\\
		&\leq 2\cdot e^{-\frac{2N\cdot\(\frac{\beta}{16}\)^2}{(1-0)^2}} \notag\\
		&\leq 2\cdot e^{-2\cdot\frac{32}{\beta^2}\cdot\log\(\frac{32\abs{S_r}}{\delta_2}\) \cdot\(\frac{\beta}{16}\)^2} \notag\\
		&\leq \frac{\delta_2}{16\abs{S_r}}. \notag
	\end{align}
	By union bound, with probability $1-\frac{\delta_2}{16}$, if there exists some $x\in R_1$ in set $S_r$, the algorithm detects and removes it.
	In other words, if \textsc{Test} terminates without restarting Algorithm~\ref{alg:boost-comp}, the probability mass $\D[R_1]\leq\frac{\epsilon}{8}$ with probability $1-\frac{\delta_2}{8}$. Therefore, if Algorithm~\ref{alg:filter}  reaches its Step~7, we have $\D[R_1 \cup R_2]\leq\frac{\epsilon}{4}$ with probability at least $1-\frac{\delta_2}{4}$.
	
	In addition, the total number of comparison tags in \test is 
	\begin{equation}
		\frac{8}{\epsilon}\log\frac{16}{\delta_2} \cdot \frac{32}{\beta^2}\cdot\log\Big(\frac{8\abs{S_r}}{\delta_2}\Big)
		= O\Big(\frac{1}{\epsilon\beta^2}\log\frac{1}{\delta_2}\log\Big(\frac{1}{\epsilon\delta_2}\Big)\Big) .
	\end{equation}
The proof is complete.
\end{proof}

\begin{lemma}[Restatement of Lemma~\ref{lem:h2}]\label{lem:h2-restate}
	Consider Phase~2 of Algorithm~\ref{alg:boost-comp}, with probability $1-\frac{\delta'}{3}$, $\err_{\D_2}(h_2)\leq\frac{\sqrt\epsilon}{2}$.
\end{lemma}
\begin{proof}
	We acknowledge that some of the deductions in this proof follow directly from that of~\cite{awasthi2017efficient} and are included for completeness.
	
	Let $R := R_1 \cup R_2$.
	By Lemma~\ref{lem:test} and $\delta_2=\frac{\delta'}{12}$, with probability at least $1-\frac{\delta'}{12}$,  $\D[R] \leq\frac{\epsilon}{4}$.
	We first argue that for any $x\notin R$, \filter does a good job to simulate $\D_I$. 
	Consider distribution $\D'$ that has equal probability on the distributions induced by $\bar{W}_I$ and $\bar{W}_C$ and let $d'(x)$ denote the density of point $x$ in this distribution. Likewise let $d_2(x)$ be the density of points in $\D_2$. 
	We want to show that for any $x\notin R$, $d'(x)=\Theta(d_2(x))$. 
	
	Recall that $\err_{\D}(h_1)=\Theta(\frac{\sqrt{\epsilon}}{2})$. 
	Let $d(x)$, $d_C(x)$, and $d_I(x)$ be the density of instance $x$ in distributions $\D$, $\D_C$, and $\D_I$, respectively. Note that, for any $x$ such that $h_1(x) = h^*(x)$, we have $d(x) = d_C(x) (1-  \frac 12\sqrt{\epsilon})$. Similarly, for any $x$ such that $h_1(x) \neq h^*(x)$, we have $d(x) = d_I(x)\frac 12 \sqrt{\epsilon}$.
	  $N_C(x)$, $N_I(x)$, $M_C(x)$ and $M_I(x)$  be the number of occurrences of $x$ in the sets $S_C$, $S_I$, $\bar{W_C}$ and $\bar{W_I}$, respectively.  
	For any $x$, there are two cases:
	
	\medskip
	\noindent{If $h_1(x) = h^*(x)$:} Then, there exist absolute constants $c_1$ and $c_2$ according to Lemma~\ref{lem:SIsize}, such that
	\begin{align}
	d'(x)&= \frac 12 \E\left[ \frac{M_C(x)}{|\bar{W_C}|} \right] \geq \frac{\E[M_C(x)]}{ c_1 \cdot  m_{\sqrt \epsilon} } \geq \frac{\E[N_C(x)]}{c_1 \cdot  m_{\sqrt \epsilon} } 
	= \frac{| S_C | \cdot d(x) }{c_1 \cdot  m_{\sqrt \epsilon} } \notag\\
	&= \frac{| S_C | \cdot d_C(x) \cdot (1- \frac 12 \sqrt{\epsilon}) }{c_1 \cdot  m_{\sqrt \epsilon} } 
	\geq c_2 d_C(x) = \frac{c_2 d_2(x)}{2}, \notag
	\end{align}
	where the second and sixth transitions are by the sizes of $\bar{W_C}$ and $|S_C|$ and the third transition is by the fact that if $h(x) = h^*(x)$, $M_C(x) > N_C(x)$.
	
	\noindent{If $h_1(x) \neq h^*(x)$:} Then, there exist absolute constants $c'_1$ and $c'_2$ according to Lemma~\ref{lem:SIsize}, such that
	\begin{align}
	d'(x) &= \frac 12 \E\left[ \frac{M_I(x)}{|\bar{W_I}|} \right] \geq \frac{\E[M_I(x)]}{c'_1 \cdot  m_{\sqrt \epsilon} }
	\geq \frac{\E[N_I(x)]}{c'_1 \cdot  m_{\sqrt \epsilon} } 
	\geq \frac{ \frac 1 2~  d(x) | S_2|}{c'_1 \cdot  m_{\sqrt \epsilon} } \notag\\
	&= \frac{ \frac 1 2~  d_I(x) \frac 12 \sqrt{\epsilon} \cdot | S_2| }{c'_1 \cdot  m_{\sqrt \epsilon} } 
	\geq  c'_2 d_I(x) = \frac{c'_2 d_2(x)}{2}, \notag
	\end{align}
	where the second and sixth transitions are by the sizes of $\bar{W_I}$ and $|S_2|$,
	the third transition is by the fact that if $h(x) \neq h^*(x)$, $M_I(x) > N_I(x)$, and the fourth transition holds by  Lemma~\ref{lem:composite-Filter}.
	
	For $x\in R$, we have $\D[R] \leq \frac{\epsilon}{4}$. Therefore, $\D_2[R]\leq\frac{\sqrt\epsilon}{4}$ because $\D=\frac{\sqrt{\epsilon}}{2}\D_I + (1-\frac{\sqrt{\epsilon}}{2})\D_C$ and $\D_2 = \frac{1}{2}\D_I + \frac{1}{2}\D_C$. 
	As a result, except for a $\frac{\sqrt\epsilon}{4}$ fraction under $\D_2$, $\forall x$ $d'(x)\geq \Theta(d_2(x))$, meaning that $\D'$ is a good simulation of $\D_2$. 	By Lemma~\ref{lem:SIsize}, since $\abs{\bar{W}}=\Theta(n_{\sqrt\epsilon,\frac{\delta'}{12}})$, applying the super-sampling lemma (Lemma~\ref{lem:super}) we know that $\err_{\D_2}(h_2) \leq O(\frac{\sqrt\epsilon}{2})$ with probability $1-\frac{\delta'}{3}$.	
\end{proof}

\begin{theorem}[Restatement of Theorem~\ref{thm:main_bothAlphaBeta}]\label{thm:main_bothAlphaBeta_restate}
Given any $\alpha, \beta, \epsilon, \delta \in (0, 1)$, assume $\eta < \frac{\min(\alpha, \beta)}{16}$. Denote $\Talpha = \log_{\frac{1}{1-\eta}} (\frac{\alpha-\alpha\eta}{\alpha - \eta})$ and $\Tbeta = \log_{\frac{1}{1-\eta}} (\frac{\beta-\beta\eta}{\beta - \eta})$. The following holds with probability $1-\delta$. There exists an algorithm (Algorithm~\ref{alg:boost-comp}) that runs in $O(\poly(d, {\Talpha}, {\Tbeta},\frac{1}{\epsilon}))$ time and returns a hypothesis $h: \X \rightarrow \Y$ with $\err_\DX(h)\leq\epsilon$. In addition, $m_L = \frac{\Talpha + \Tbeta}{\alpha^2}\cdot \tilde{O}\big( \log\frac{d + (\Talpha + \Tbeta)\frac{1}{\delta}}{\epsilon} \big)$, $m_C = \frac{1}{\beta^2}\cdot \tilde{O}\big( \big(\Talpha + \Tbeta\big)^{\frac{1001}{1000}} \cdot n_{\epsilon,\delta} \big)$, and $m_V = O(\Talpha + \Tbeta)$. Therefore, $\labeloverhead = \tilde{o}\big( \frac{\Talpha + \Tbeta}{\alpha^2}\big)$ and $\compareoverhead = \tilde{O}_{\delta}\big( \frac{1}{\beta^2}\big(\Talpha+\Tbeta\big)^{\frac{1001}{1000}}\big)$ when $\epsilon\in\(0,(\log d)^{-4}\)$. In particular, when $\alpha$ and $\beta$ are constants, $\labeloverhead = o(1)$, $\compareoverhead = O(1)$, and $m_V = O(1)$.
\end{theorem}
\begin{remark}
	The above theorem presents tighter bounds for $m_L,m_C,\labeloverhead,\compareoverhead$ than that in Theorem~\ref{thm:main_bothAlphaBeta}. Notice that $\frac{1}{\alpha}=O(\Talpha)$ and $\frac{1}{\beta}=O(\Tbeta)$. Hence, Theorem~\ref{thm:main_bothAlphaBeta} in our main paper is more general.
	In addition, the quantities $\Talpha,\Tbeta,\frac{1}{\alpha},\frac{1}{\beta},\eta$ are indepedent of the dimension $d$, the desired error rate $\epsilon$, and the confidence parameter $\delta$. Hence, when consider $m_V$ and the overheads, our algorithm is query-efficient.
\end{remark}
\begin{proof}
	Following Corollary~\ref{coro:h1-restate}, Lemma~\ref{lem:restart_restate} and \ref{lem:h2-restate}, if the main algorithm outputs a hypothesis $h$ without restarting, we have the following guarantees with probability $1-\delta$: (i) $\err_{\D}(h_1)\leq\frac{\sqrt\epsilon}{2}$; $\err_{\D_2}(h_2)\leq\frac{\sqrt\epsilon}{2}$;  $\err_{\D_3}(h_3)\leq\frac{\sqrt\epsilon}{2}$ hold simultaneously; (ii) the algorithm queries at most $O(\Talpha)$ verified labels and $O(\Tbeta)$ verified comparison tags from $\expert$. 
	Applying Theorem~\ref{thm:boost}, the hypothesis $\hat{h}$ returned by Algorithm~\ref{alg:boost} is such that $\err_{\D}(\hat{h})\leq\epsilon$ with probability $1-\delta$. 
	
	In Lemma~\ref{lem:restart_restate}, we show that the success of our pruning scheme crucially relies on a condition that $\eta<\frac{\alpha}{c}$ and $\eta<\frac{\beta}{c}$ for some large enough constant $c>0$. Set $c=16$. We require $\eta<\frac{\min(\alpha,\beta)}{16}$ for the guarantees in Theorem~\ref{thm:main_bothAlphaBeta} to hold. Moreover, $m_V = O(\Talpha+\Tbeta)$.
	
	It remains to show the label and comparison compleixty, $m_L, m_C$, and corresponding overheads $\labeloverhead, \compareoverhead$.
	The label complexity follows from Proposition~\ref{prop:complabel-restate} by setting $n=n_{\sqrt{\epsilon},\delta'}$ and the fact that main algorithm is only restarted by a total of $O(\Talpha+\Tbeta)$ times. Note that $O(\frac{1}{\delta_1})=O(\frac{1}{\delta_2})= O(\frac{1}{\delta'}) = O((\Talpha+\Tbeta)\cdot\frac{1}{\delta})$. We have
	\begin{align}
		m_L &= (\Talpha+\Tbeta) \cdot O\bigg(\frac{\log n_{\sqrt{\epsilon},\delta'}}{\alpha^2}\cdot \log\log n_{\sqrt{\epsilon},\delta'} \bigg) \notag\\
		&= \frac{\Talpha+\Tbeta}{\alpha^2} \cdot \tilde{O}\bigg( \log\frac{d + (\Talpha+\Tbeta)\frac{1}{\delta}}{\epsilon} \bigg). \notag
	\end{align}
	On the other hand, the comparison complexity is a summation of the results from Lemma~\ref{lem:test-restate} and that of~\cite{zeng2022crowd},
	\begin{align}
		&m_C = (\Talpha+\Tbeta) \cdot O\bigg(\frac{n_{\sqrt{\epsilon},\delta'}}{\beta^2}\cdot \log^2 n_{\sqrt{\epsilon},\delta'} + n_{\epsilon,\delta'} + \frac{1}{\epsilon\beta^2}\log\Big(\frac{1}{\delta'}\Big)\log\Big(\frac{1}{\epsilon\delta'}\Big) \bigg) \notag \\
		&= \frac{\Talpha+\Tbeta}{\beta^2} \cdot O\bigg( \frac{d \log\frac{1}{\epsilon} + \big((\Talpha+\Tbeta)\frac{1}{\delta}\big)^{\frac{1}{1000}} + \log\big((\Talpha+\Tbeta)\frac{1}{\delta}\big)\cdot\big(\log\frac{1}{\epsilon}+\log(\Talpha+\Tbeta)\frac{1}{\delta}\big)}{\epsilon} \bigg) \notag \\	
		&= \frac{1}{\beta^2} \cdot \tilde{O}\bigg( (\Talpha+\Tbeta)^{\frac{1001}{1000}} \cdot \frac{d+\big(\frac{1}{\delta}\big)^{\frac{1}{1000}}}{\epsilon} \notag\bigg) = \frac{1}{\beta^2}\cdot \tilde{O}\big( \big(\Talpha + \Tbeta\big)^{\frac{1001}{1000}} \cdot n_{\epsilon,\delta} \big)
	\end{align}

	These in allusion to $m_{\epsilon,\delta} = K \cdot \big(\frac{1}{\epsilon}(d \log(1/\epsilon)+\log(1/\delta)\big)$ immediately give the overheads as follows:
	\begin{align}
		\labeloverhead &= O\bigg( \frac{\Talpha+\Tbeta}{\alpha^2} \cdot \frac{ \log n_{\sqrt{\epsilon},\delta'} }{m_{\epsilon,\delta}} \cdot  \log {\log n_{\sqrt{\epsilon},\delta'}} \bigg) \notag\\
		&\leq \frac{\Talpha+\Tbeta}{\alpha^2} \cdot \frac{\epsilon}{d+ \log(1/\delta)} \cdot \tilde{O}\bigg( \log\frac{d + \big(\Talpha+\Tbeta\big)\frac{1}{\delta}}{\epsilon} \bigg) \notag\\
		&= \frac{\Talpha+\Tbeta}{\alpha^2}  \log\Big(\Talpha+\Tbeta\Big) \cdot \frac{\epsilon}{d} \cdot \tilde{O}\Big( \log\frac{d}{\epsilon} \Big), \notag
	\end{align}
	and
	\begin{align}
		\compareoverhead &= \frac{\Talpha+\Tbeta}{\beta^2}\cdot O\Bigg(\frac{n_{\sqrt{\epsilon},\delta'}}{m_{\epsilon, \delta}} \cdot \log^2 n_{\sqrt{\epsilon},\delta'} + \frac{n_{\epsilon,\delta'}}{m_{\epsilon,\delta}} + \frac{\frac{1}{\epsilon}\log\big(\frac{1}{\delta'}\big)\log\big(\frac{1}{\epsilon\delta'}\big)}{m_{\epsilon,\delta}} \Bigg) \notag\\
		&\leq \frac{\Talpha+\Tbeta}{\beta^2}\cdot O\Bigg(\sqrt{\epsilon}\cdot \frac{d\log\frac{1}{\epsilon} + \big(\big(\Talpha+\Tbeta\big)\frac{1}{\delta}\big)^{\frac{1}{1000}} }{d\log\frac{1}{\epsilon} + \log\frac{1}{\delta}} \cdot  \log^2 \bigg( \frac{d + \big(\Talpha+\Tbeta\big)\frac{1}{\delta}}{\epsilon} \bigg) \notag\\
		&\quad\quad + \frac{d\log\frac{1}{\epsilon} + \big(\big(\Talpha+\Tbeta\big)\frac{1}{\delta}\big)^{\frac{1}{1000}} }{d\log\frac{1}{\epsilon} + \log(\frac{1}{\delta})} + \frac{\log\big(\big(\Talpha+\Tbeta\big)\frac{1}{\delta}\big)\log\big(\big(\Talpha+\Tbeta\big)\frac{1}{\epsilon\delta}\big)}{d\log\frac{1}{\epsilon}+\log\frac{1}{\delta}} \Bigg) \notag\\ 
		&\leq \frac{1}{\beta^2}\Big(\Talpha+\Tbeta\Big)^{\frac{1001}{1000}}\log^2\Big(\Talpha+\Tbeta\Big)\cdot O\Bigg(\sqrt{\epsilon}\cdot \frac{d\log\frac{1}{\epsilon} + \big(\frac{1}{\delta}\big)^{\frac{1}{1000}} }{d\log\frac{1}{\epsilon} + \log\frac{1}{\delta}} \cdot  \log^2 \bigg( \frac{d + \frac{1}{\delta}}{\epsilon} \bigg) \notag\\
		&\quad\quad  + \frac{d\log\frac{1}{\epsilon} + \big(\frac{1}{\delta}\big)^{\frac{1}{1000}} }{d\log\frac{1}{\epsilon} + \log(\frac{1}{\delta})} + \frac{\log\big(\frac{1}{\delta}\big)\big(\log\frac{1}{\epsilon}+\log\frac{1}{\delta}\big)}{d\log\frac{1}{\epsilon}+\log\frac{1}{\delta}} \Bigg). \notag
	\end{align}
	
	Recall that by the definition of $n_{{\epsilon}, \delta}$ and $m_{\epsilon, \delta}$ in Section~\ref{sec:setup}, we have $m_{\epsilon, \delta} = \Theta(n_{\epsilon, \delta}) = \Theta(\frac{d}{\epsilon} \log\frac{1}{\epsilon})$ when $\delta$ is a constant. In this case, we can see that
	\begin{equation*}
	\compareoverhead \leq \frac{1}{\beta^2}\Big(\Talpha+\Tbeta\Big)^{\frac{1001}{1000}}\log^2\Big(\Talpha+\Tbeta\Big)\cdot O_{\delta}\Big(\sqrt{\epsilon}\cdot \log^2 \frac{d}{\epsilon} + 2 \Big).
	\end{equation*}
	When $\epsilon\in\(0,(\log d)^{-4}\)$, $\compareoverhead = \tilde{O}_{\delta}\Big(\frac{1}{\beta^2}\big(\Talpha+\Tbeta\big)^{\frac{1001}{1000}}\Big)$.
	In addition, $\labeloverhead = o\Big( \frac{\Talpha+\Tbeta}{\alpha^2}\cdot \log(\Talpha+\Tbeta)\Big)$, because $\labeloverhead$ goes to $0$ when $\epsilon$ goes to $0$.
\end{proof}

%

\section{Useful Lemmas}\label{subsec:useful_lemmas}


\begin{theorem}[Boosting, \citet{schapire1990strength}]\label{thm:boost}
For any $p<\frac{1}{2}$ and distribution $\D$, consider three classifiers $h_1(x)$, $h_2(x)$, $h_3(x)$ satisfying the following. 1) $\err_{\D}(h_1) \leq p$; 2) $\err_{\D_2}(h_2) \leq p$ where $\D_2 := \frac{1}{2}\D_C + \frac{1}{2}\D_I$, $\D_C$ denotes the distribution $\D$ conditioned on $\{x: h_1(x)=h^*(x)\}$, and $\D_I$ denotes $\D$ conditioned on $\{x : h_1(x)\neq h^*(x)\}$; 3) $\err_{\D_3}(h_3) \leq p$ where $\D_3$ is $\D$ conditioned on $\{x: h_1(x)\neq h_2(x)\}$. Then $\err_{\D}(\maj(h_1, h_2, h_3)) \leq 3p^2-2p^3$.
\end{theorem}

\begin{lemma}[Robust Super-Sampling Lemma, {Lemma 4.12} in~\citet{awasthi2017efficient}]\label{lem:super}
	Given a hypothesis class $\H$ consider any two discrete distributions $\D$ and $\D'$ over $\X$ such that except for an $\epsilon$ fraction of the mass under $\D$, we have that for all $x$, $d'(x)\geq c\cdot d(x)$ for an absolute constant $c>0$ and both distributions are labeled according to $h^*\in\H$. There exists a constant $c'>1$ such that for any $\epsilon,\delta$, with probability $1-\delta$ over a labeled sample set $S$ of size $c'm_{\epsilon,\delta}$ drawn from $\D'$, $\A_\H(S)$ has error of at most $2\epsilon$ with respect to $D$.
\end{lemma}


\begin{lemma}[{Lemma 14} in~\citet{zeng2022crowd}]\label{lem:composite-Filter}
	Consider Algorithm~\ref{alg:filter}. Assume that the subset $U$ is correctly labeled. Consider any given instance $x\in S_2$ except for an $\frac{\epsilon}{2}$-fraction, we have the following guarantee. If $h(x)=h^*(x)$, it will be added to $S_I$ with probability at most $\frac14\sqrt{\epsilon}$; if $h(x)\neq h^*(x)$, it goes to $S_I$ with probability at least $\frac{4}{7}$. 
\end{lemma}
\begin{proof}
	First, for any instance $x \in S_{\text{in}}\cup U$, the guarantee follows from that of \complabel. 
	
	Now consider any instance $x\in S\backslash\{S_{\text{in}}\cup U\}$.
	Recall that Lemma~\ref{lem:test} guarantees the probability mass $\D[R_1\cup R_2] \leq \frac{\epsilon}{4}$ with probability $1-\frac{\delta_2}{4}$. Given that the input size of $S$ is $\Theta(n_{\epsilon,\delta_2})$, with probabiltiy at least $1-\frac{\delta_2}{12}$, the fraction of the instances $x\in S_2$ that falls in region $R_1\cup R_2$ is less than $\frac{\epsilon}{2}$. Therefore, at least $1-\frac{\epsilon}{2}$ fraction of the instances in $S$ falls outside $R_1\cup R_2$ such that $\smaj_{\PC}(x,x^-)\geq1-\frac{\beta}{2}$ and $\smaj_{\PC}(x,x^+)\geq1-\frac{\beta}{2}$ (Eq.~\eqref{eq:R1_R2}). For these instances, we can consider that the majority is correct.
	Now if $x$ is outside the interval $[x^-, x^+]$, when $\beta<0.7$, $\smaj_{\PC}(x,x')\geq \frac12+\Theta(1)$ which suffices for the filtering scheme from~\citet{awasthi2017efficient} to work; when $\beta\geq0.7$, lemma follows from that of~\citet{zeng2022crowd}. If $x$ falls into the interval $[x^-, x^+]$, the probability that it will be added to $S_I$ but $h(x)=h^*(x)$, or it will not be added to $S_I$ but $h(x)\neq h^*(x)$ are both less than $\frac14\sqrt{\epsilon}$. The proof is complete.
\end{proof}


\begin{lemma}[{Lemma 4.7} in~\citet{awasthi2017efficient}] \label{lem:SIsize}
	Consider Algorithm~\ref{alg:boost}. With probability at least $1-\exp(-\Omega(m_{\sqrt{\epsilon},\delta'}))$, $\bar{W}_I$, $\bar{W}_C$ and $S_I$ all have size $\Theta(m_{\sqrt{\epsilon},\delta'})$.
	Consider Algorithm~\ref{alg:boost-comp}. With probability at least $1-\exp(-\Omega(n_{\sqrt{\epsilon},\delta'}))$, $\bar{W}_I$, $\bar{W}_C$ and $S_I$ all have size $\Theta(n_{\sqrt{\epsilon},\delta'})$.
\end{lemma}
\begin{proof}
	Notice that in both Algorithm~\ref{alg:boost} and \ref{alg:boost-comp}, we have  the input sample size $m_{\epsilon,\delta'}, n_{\epsilon,\delta'}$  lower bounded by $\Omega(m_{\epsilon,\delta'})$, the lemma follows from {Lemma 4.7} of~\citet{awasthi2017efficient}.
\end{proof}


\begin{proposition}[Proposition 4 in \citet{zeng2022crowd}]\label{prop:naive_complabel}
Suppose $\alpha\geq0.7,\beta\geq0.7$. Consider the \complabel algorithm, i.e. Algorithm~\ref{alg:label}. If $\abs{S} \geq  (\frac{3}{\delta})^{1/1000}$, then with probability at least $1-\delta$, it correctly sorts and labels all the instances in $S$. The label complexity is $O\big(\log\abs{S}\cdot\log{\log\abs{S}} \big)$, and the comparison complexity is given by $O\big(\abs{S}\cdot\log^2\abs{S}\big)$.
\end{proposition}

\begin{theorem}[Theorem 8 in \citet{zeng2022crowd}]\label{thm:base}
Suppose $\alpha\geq0.7,\beta\geq0.7$. With probability $1-\delta$, Algorithm~\ref{alg:boost-comp} runs in time $\poly(d,\frac{1}{\epsilon})$ and returns a classifier $h\in\H$ with error rate $\err_{\D}({h})\leq\epsilon$.
In addition, the label complexity is $O\Big(\log n_{\sqrt\epsilon,\delta} \cdot \log\log n_{\sqrt\epsilon,\delta} \Big)$, and the comparison complexity is $O\( n_{\sqrt{\epsilon},\delta}\log^2 n_{\sqrt{\epsilon},\delta}+n_{\epsilon,\delta}\)$.
\end{theorem}

\end{document}